\documentclass{article} %
\ifdefined\pdfminorversion\pdfminorversion=7\fi

\usepackage[]{conference,times}
\iclrfinalcopy 

\usepackage{amsmath,amsfonts,bm}

\def\1{\bm{1}}

\DeclareMathAlphabet{\mathsfit}{\encodingdefault}{\sfdefault}{m}{sl}
\SetMathAlphabet{\mathsfit}{bold}{\encodingdefault}{\sfdefault}{bx}{n}

\newcommand{\E}{\mathbb{E}}

\newcommand{\R}{\mathbb{R}}

\newcommand{\KL}{\mathrm{KL}}
\newcommand{\Var}{\mathrm{Var}}

\DeclareMathOperator{\arcsinh}{arcsinh}

\usepackage[utf8]{inputenc} %
\usepackage[T1]{fontenc}    %
\usepackage{hyperref}       %
\usepackage{url}            %
\usepackage{booktabs}       %
\usepackage{amsfonts}       %
\usepackage{nicefrac}       %
\usepackage{xcolor}         %

\usepackage{amsmath,amssymb,amsthm,mathtools}
\usepackage{thmtools}
\usepackage{cleveref}

\usepackage{microtype}      %

\usepackage{multirow}
\usepackage{stmaryrd}
\usepackage{wrapfig}

\usepackage{pgfplots}
\usepackage{pgfplotstable} %
\usepgfplotslibrary{groupplots}
\pgfplotsset{compat=1.18}

\usepackage[np,autolanguage]{numprint}
\nprounddigits{2}

\usepackage{tikz}
\usetikzlibrary{positioning}

\declaretheorem[name=Theorem]{theorem}
\declaretheorem[name=Corollary, sibling=theorem]{corollary}
\declaretheorem[name=Lemma, numberwithin=section]{lemma}
\declaretheorem[name=Proposition, sibling=lemma]{proposition}

\declaretheorem[name=Discussion, style=definition, numbered=no]{discussion}

\usepackage{algorithm}
\usepackage{algorithmic}

\usepackage{mathrsfs} %
\usepackage{xspace}
\makeatletter
\DeclareRobustCommand\onedot{\futurelet\@let@token\@onedot}
\def\@onedot{\ifx\@let@token.\else.\null\fi\xspace}
\def\iid{{i.i.d}\onedot}
\def\eg{{e.g}\onedot} 
\def\ie{{i.e}\onedot}

\makeatother

\newcommand{\Prob}{\mathbb{P}}
\newcommand{\A}{\mathcal{A}}
\newcommand{\M}{\mathcal{M}}
\newcommand{\N}{\mathcal{N}}
\newcommand{\X}{\mathcal{X}}
\newcommand{\Y}{\mathcal{Y}}
\newcommand{\F}{\mathcal{F}}

\newcommand{\onehalf}{{\textstyle{\frac{1}{2}}}}
\newcommand{\onequarter}{{\textstyle{\frac{1}{4}}}}

\newcommand{\risk}{R} %
\newcommand{\erisk}{\widehat{\risk}} %
\newcommand{\kl}{\mathrm{kl}}

\newcommand{\Renyi}{R\'enyi\xspace}
\newcommand{\renyiDiv}{D}
\newcommand{\rdpOrder}{\alpha}
\newcommand{\rdpValue}{\epsilon}
\newcommand{\pointDistribution}{\mu}

\usepackage{caption}
\title{From Privacy to Generalization: Linear Max-Information Bounds for Differentially Private Learning Algorithms}

\author{Christoph~H.~Lampert \qquad Max Cairney-Leeming \qquad Hossein Zakerinia 
\\
Institute of Science and Technology Austria (ISTA), Klosterneuburg, Austria}

\begin{document}
\maketitle

\begin{abstract}
  Understanding the relationship between \emph{generalization} and \emph{privacy}
  remains a challenge in modern machine learning theory, particularly for deep 
  networks that are trained by variants of differentially private stochastic 
  gradient descent (DP-SGD).
  In this work we make progress on this persistent open problem.
  First, we derive explicit upper bounds on the approximate max-information 
  of any algorithm that fulfills $(\epsilon,\delta)$ differential privacy 
  or \Renyi differential privacy, thereby going beyond the classical results 
  for pure $\epsilon$ differential privacy.
  Subsequently, we show even stronger guarantees for two common private learning 
  algorithms, output perturbation with the Gaussian mechanism, and streaming DP-SGD, by exploiting the structure of their internal randomization. 
  As an application of our results, we demonstrate how to obtain 
  non-vacuous PAC-Bayes generalization bounds for deep networks, 
  in which the prior distribution is learned by DP-SGD 
  instead of the classical way of choosing it in a data-independent way.
\end{abstract}

\section{Introduction}\label{sec:intro}
The question of how to formally guarantee that learning algorithms
generalize from their training set
to new data has been at the core of machine learning theory
ever since the seminal works on PAC learning in the 1970s and
1980s~\citep{vapnik1971uniform,valiant1984theory}.
In recent years, remarkable progress has been made to extend these generalization guarantees to over-parameterized models, using the PAC-Bayes
framework~\citep{mcallester1999some,seeger2002pac,catoni2007pac}, resulting in non-vacuous bounds for deep neural networks~\citep{dziugaite2017computing,zhou2018non,lotfi2024nonvacuous}.
Phrased informally, one of the core insights is that models
generalize if they do not
\emph{overfit}, \ie they extract information from the training set
but do not \emph{memorize} it.
Related considerations arise when learning from datasets that contain sensitive or personal data, \eg when training   \emph{differentially private}~\citep{DworkRoth2014,ponomareva2023dpfy} models.
To guarantee that the original training data cannot be extracted from
the trained model,
one has to constrain the amount that each data point can
influence the model
parameters, \ie preventing \emph{data memorization}.

Given their conceptual similarity, it is not surprising that the intersection of
generalization and privacy became a fruitful area for innovation. %
For example, in classic settings,
\citet{dwork2015generalization,bassily2016algorithmic}
established and \citet{jung2021new} refined generalization guarantees for
statistical and low-sensitivity queries.
\citet{bassily2014private} proved bounds for private risk minimization algorithms in convex settings, and \citet{bombari2025privacy} and
\citet{shi2026towards} for private gradient descent for random feature 
models and two-layer ConvNets, respectively.
Specifically in a PAC-Bayes setup, \citet{dziugaite2018data} demonstrated
that privately learned data-dependent distributions can take the place of the
usually data-independent prior distributions with only a controllable
additive penalty.

Unfortunately, all of the above results are inapplicable or ineffective 
for modern deep networks, which tend to be over-parameterized, non-convex 
and trained using stochastic gradient descent, so new theory is needed. One key mismatch between prior work and these networks is the lack of model- and problem-generic results for the privacy notions that these models are trained on, approximate $(\epsilon,\delta)$-DP~\citep{dwork2006gaussian}, 
and \Renyi DP~\citep{Mironov2017}, which are more amenable to analyzing the privacy guarantees of SGD than pure $\epsilon$-DP. 
However, the connection between these notions of privacy and the generalization properties of the resulting models is so far not well understood.

In this work, we make progress on this question, through 
the notion of \emph{approximate max-information}~\citep{dwork2015generalization},
which is a statistical measure of how much information a trained model contains 
about its training set and therefore is highly informative about questions of 
\emph{memorization} and \emph{generalization}.
\textbf{In \Cref{sec:general}, we prove general-purpose bounds on the approximate 
max-information for algorithms that fulfill either $(\epsilon,\delta)$ approximate DP 
or \Renyi DP.} %
To our knowledge, these are the first such guarantees that can readily be applied 
in \emph{practical} private machine learning tasks, in the sense that they hold in 
a regime where real-world deep learning models are trained, let it be in industry, 
such as Meta's recent private vision-language model \emph{DP-Cap}~\citep{sander2024differentiallyprivaterepresentationlearning},
or Google's large language model \emph{VaultGemma}~\citep{sinha2025vaultgemmadifferentiallyprivategemma};
or academia, such as recent models for private human action recognition~\citep{luo2023differentiallyprivatevideoactivity,nken2025VideoDPRP}.

Afterwards, we focus on two specific algorithms, that are particularly popular for machine learning tasks: \emph{output perturbation} using the Gaussian mechanism~\cite{dwork2006gaussian}, 
and (streaming) \emph{DP-SGD}~\citep{rajkumar2012differentially}.
\textbf{In Section~\ref{sec:maxinf-practical}, we show that by exploiting the explicit 
form of randomization invoked by these algorithms, tighter guarantees on the max-information 
are possible than what would follow just from their privacy properties.}

Finally, in Section~\ref{sec:pacbayes}, we use our results to instantiate 
PAC-Bayes bounds~\citep{mcallester1999some}, with data-dependent prior distributions learned using DP-SGD, following~\citet{dziugaite2018data,rivasplata2020pac}, and we demonstrate that the resulting additive penalty terms can be controlled explicitly from the training hyperparameters. Specifically, \textbf{we give non-vacuous generalization bounds for deep networks with up to 13 million trainable parameters}, which is a much larger than what was possible with data-independent priors. 

\section{Background}\label{sec:background}
In this section, we introduce the main concepts and notation.
For more details, see Appendix~\ref{app:extendedbackground}.

We adopt a standard setup of statistical learning, in which a
learning algorithm, $\A$, is given an \iid training set, $S\subset\mathcal{X}^n$, 
and outputs a model $y\in\Y\subseteq\R^d$ (
we identify models and their parametrizations).
We generally allow for algorithms to be \emph{stochastic}, \ie we 
treat $\A(S)$ as a random variable that takes values in $\Y$, 
according to the distribution $\text{Law}(\A(S))$. 

\paragraph{Differential Privacy.}
A stochastic algorithm, $\A:\X^n\to\Y$, is called \emph{approximate
differentially private}, $(\epsilon,\delta)$-DP, for some
$\epsilon>0$ and $\delta\in[0,1]$, if
\begin{equation}
  \forall O\subset\Y:\quad \Prob\{\A(S)\in O\}\leq
  e^{\epsilon}\Prob\{\A(S')\in O\} + \delta,
\end{equation}
for all datasets $S,S'$ that are \emph{neighboring}, \ie identical
except for a single data point\footnote{In this work, we use the \emph{replace one}
neighborhood relation. %
Adapting our results to other relations, such as \emph{add/remove} would be straightforward and only result in slightly different constants.}~\citep{dwork2006gaussian}   
If an algorithm is $(\epsilon,\delta)$-DP for $\delta=0$, we call it
\emph{purely differentially private}, $\epsilon$-DP. %

A popular alternative notion of privacy in machine learning contexts 
is \Renyi DP (RDP)~\citep{Mironov2017}. 
For any $\rdpOrder>1$ and $\rdpValue\geq 0$, the algorithm $\A$ is 
called \emph{$(\rdpOrder,\rdpValue)$-RDP}, if 
$D_\rdpOrder(\A(S)\|\A(S'))\leq \rdpValue$
for any two neighboring datasets $S$ and $S'$, where $D_\rdpOrder(P\|Q):=\frac{1}{\rdpOrder-1}\log\E_{x\sim Q}\bigl(\frac{P(x)}{Q(x)}\bigr)^{\rdpOrder}$ is the \Renyi divergence of order $\rdpOrder$. 
Note that the relation between both notions of privacy is asymmetric: 
$(\rdpOrder,\rdpValue)$-RDP implies $(\rdpValue+\frac{\log(1/\delta)}{\rdpOrder-1})$-DP 
for any $\delta\in(0,1)$, but from a single $(\epsilon,\delta)$-DP guarantee, 
no guarantees on RDP follow~\citep{Mironov2017}.

\paragraph{Approximate Max-Information.}
A statistical quantity, related to privacy, but specifically tailored to 
studying problems of generalization is the \emph{(approximate) max-information}~\citep{dwork2015generalization}, which measures how much statistical information an algorithm's output, $\A(S)\in\Y$, contains about its input data, $S\in\X^n$.
Formally, for any $\beta\in[0,1)$, we denote by 
\begin{equation}
    D^{\beta}_{\infty}(V\|W) := \log \sup\limits_{O\subseteq\Omega: \Pr\{V\in O\}>\beta}\frac{\Prob\{V\in O\}-\beta}{\Pr\{W\in O\}}
\end{equation}
the \emph{approximate max-divergence} of two random variables $V$ and $W$ on a common domain $\Omega$, and we define the approximate max-information
\begin{equation}
    I^{\beta}_{\infty}\big(\A(S),S\big) := D^{\beta}_{\infty}\bigl((\A(S),S)\|\A(S)\times S\bigr),
\label{eq:maxinf_def}
\end{equation}
where $\A(S)\times S$ denotes the random variable distributed according to $\text{Law}(\A(S))\times\text{Law}(S)$, \ie the product of independent marginal.
Note that we work with \emph{natural} logarithms, unlike the original definition in~\citet{dwork2015generalization}, 
and so avoid any $\log e$ constants 
that appear in the original works. 

The max-information can act as a tool for studying, \eg, \emph{memorization} or 
\emph{unlearning} in a principled way~\citep{gupta2021adaptive}.
Similarly, knowing the max-information of a learning algorithm allows the 
derivation of bounds on its generalization abilities~\citep{dziugaite2018data,rivasplata2020pac}.
A link between max-information and privacy was established in \citet[Theorem 8]{dwork2015generalization}: 
for any $\epsilon$ differentially private algorithm $\mathcal{A}$ that is 
trained on an \iid dataset of size $n$, it holds that
\begin{align}
  I^{\beta}_{\infty}(\A(S),S) &\leq \frac12 n\epsilon^2
  +\epsilon\sqrt{\frac{n}{2}\log\frac{2}{\beta}}.
  \label{eq:maxinf_from_puredp}
\end{align}
Unfortunately, $\epsilon$-DP is a very strict condition that most 
practical machine learning algorithms do not fulfill. Therefore, 
one of the central open questions in the field is 
how~\eqref{eq:maxinf_from_puredp} can be extended 
to more practical relaxed notions of differential privacy.
\citet[Theorem 3.1]{rogers16maxinfo}
discussed a generalization of \eqref{eq:maxinf_from_puredp}
to $(\epsilon,\delta)$-DP algorithms.
However their focus was on establishing asymptotic guarantees and
holds only for certain combinations of $\epsilon$, $\delta$ and 
$\beta$, so the result is not applicable for practical tasks, 
where constants matter and one would like to choose these parameters 
freely.

\section{From Privacy to Approximate Max-Information}\label{sec:general}
We now present our first set of results: explicit upper bounds on the 
approximate max-information of differentially private learning algorithms.
Specifically, in \Cref{thm:maxinf_from_epsdeltadp} we present a new result 
for $(\epsilon,\delta)$-DP that at the same time substantially tightens the 
classic result~\eqref{eq:maxinf_from_puredp} for pure $\epsilon$-DP (\ie $\delta=0$).
In \Cref{thm:maxinf_from_renyi}, we establish guarantees on the max-information 
of \Renyi-DP algorithms.
Throughout this section, we assume that $\mathcal{A}:\X^n\to\R^d$ is an arbitrary 
stochastic algorithm, operating on an \iid dataset, $S$, of size $n$. 
Additional assumptions, in particular about $\A$'s privacy properties, 
will be stated in the individual theorems. The proofs are deferred to Appendix~\ref{app:proofs}.

\begin{theorem}\label{thm:maxinf_from_epsdeltadp}
Assume that $\A$ is $(\epsilon,\delta)$-DP for some $\epsilon\geq 0$ and $\delta\in[0,1]$. 

1) If $\beta\in[n\delta,1)$, it holds with $\beta_0:=\beta-n\delta\geq 0$: 
\begin{align} 
I^{\beta}_{\infty}(\A(S),S) &\leq \min\Bigl\{ n\epsilon,\ \frac18n\epsilon^2 +\epsilon\sqrt{\frac n2 \log \frac{1}{\beta_0}} + \frac{n \delta}{\beta_0}c(\epsilon)\Bigr\}, \label{eq:max-inf-epsilondelta}
\end{align}
where $c(\epsilon) = e^{2\epsilon} + \epsilon - 1 + e^{\sqrt{\frac{32 \log(1/\beta_0)}{n}}}$
and the second term in the $\min$ only occurs if $\beta_0>0$.   

2) If $\beta\in(0,n\delta)$, no finite upper bound on $I^{\beta}_\infty(\A(S),S)$ in terms of $n$, $\epsilon$, $\delta$ and $\beta$ is possible. 
\end{theorem}
\hspace{0.3em}
\begin{discussion}
    \Cref{thm:maxinf_from_epsdeltadp} generalizes and improves over the classic result~\eqref{eq:maxinf_from_puredp}.
    For $\delta=0$, \ie pure $\epsilon$-DP, statement 1) offers an improvement by a factor of $4$ in 
    the first term, which is often the dominant one.
    While not a change in rate---which would not be possible, because $O(n\epsilon^2)$ is known to be optimal~\citep{rogers16maxinfo}---it nevertheless can provide 
    a substantial improvement when numerical values of $I^{\beta}_{\infty}$ are required, such as for 
    PAC-Bayes generalization guarantees~\citep{dziugaite2017computing}.
    When $\delta=0$, it is possible to choose $\beta=0$, in which case we recover the known upper bound $n\epsilon$ on the \emph{pure} max-information from \citet[Theorem 7]{dwork2015generalization}.
    
    For $\delta>0$, Eq. \eqref{eq:max-inf-epsilondelta} goes beyond ~\citet[Theorem 3.1]{rogers16maxinfo}, in that it provides 
    expressions with explicit constants, that---in particular--- converge to the pure-DP 
    result for $\delta\to 0$. In contrast to the earlier work, our bound also does not 
    contain terms of the order $\frac{\delta}{\epsilon}$ and remains valid for 
    $\epsilon\to 0$ without requiring $\delta\to 0$ at the same time.
    
    Statement 2) makes the observation precise that $\beta$ cannot be too small if the approximate
    max-information is meant to be controllable, as seen in~\citep{rogers16maxinfo}.
    Specifically, $\beta\geq n\delta$ is necessary to achieve finite bounds on $I^{\beta}_\infty(\A(S),S)$,
    and $\beta> n\delta$ for non-trivial ones.
\end{discussion}

The following result provides an analogous guarantee for \Renyi DP algorithms.
\begin{theorem}\label{thm:maxinf_from_renyi}
Assume $\A$ is $(\rdpOrder,\rdpValue)$-RDP for some $\rdpOrder>1$ and $\rdpValue>0$. Then, 
it holds for any $\beta\in(0,1)$:
\begin{equation} 
I^{\beta}_{\infty}(\A(S),S) \leq \min\Bigl\{ n\Bigl(\rdpValue+\frac{\log(1/\beta)}{\rdpOrder-1}\Bigr), \ 
v + 2\sqrt{v\log(2/\beta)} + \frac{\log(2/\beta)}{\rdpOrder-1},
\Bigr\}\label{eq:max-inf-rdp}
\end{equation}
with $v=\frac{n}{8}\bigl(\rdpValue + \frac{\log(8/\beta)}{\rdpOrder-1} \bigr)^2$.

\end{theorem} 
\begin{discussion}
    Like the bound in \Cref{thm:maxinf_from_epsdeltadp}, the right hand side 
    of \eqref{eq:max-inf-rdp} has two parts, which apply in different regimes. 
    The central quantity is the \emph{effective privacy}, $\tilde\rdpValue := \rdpValue+\frac{\log(c/\beta)}{\rdpOrder-1}$, which resembles the conversion rule from RDP 
    to $(\epsilon,\delta)$-DP and appears in both parts, only 
    with different constants $c$. 
    The first part of the bound is simply linear, $n\tilde\rdpValue$, thereby 
    providing an analog to the first term in~\eqref{eq:max-inf-epsilondelta}.
    The second part of the bound is of the order $O(n\tilde\rdpValue^2 + \sqrt{n}\tilde\rdpValue)$
    and can therefore be much smaller than the first in the high-privacy regime,
    where $n$ is large and $\tilde\rdpValue\ll 1$. 
    For large $\tilde\rdpValue$, whether due to weak privacy  (large $\rdpValue$)
    or a very small value of $\beta$, the quadratic second term grows quickly, and 
    the first term provides a tighter bound.
    In contrast to \Cref{thm:maxinf_from_epsdeltadp}, there is no restriction on 
    the valid range of $\beta>0$. This is due to the fact that any single $(\alpha,\epsilon)$-RDP guarantee actually ensures $(\epsilon,\delta)$-DP guarantees for all values of $\delta$, 
    only with adjusted $\epsilon$, including sufficiently small ones.

An alternative path to prove a bound on the max-information of RDP algorithms 
would be to rely on the conversion to $(\epsilon,\delta)$-DP and evaluate \Cref{thm:maxinf_from_epsdeltadp}. Typically, this would lead to worse guarantees, though, because of the large last term in \eqref{eq:max-inf-epsilondelta} that has no equivalent in \eqref{eq:max-inf-rdp}.
\end{discussion}

\section{Approximate Max-Information of Private Learning Algorithms}\label{sec:maxinf-practical}
So far we presented results hold for all algorithms that fulfill certain privacy definitions.
Next, we study two specific algorithms for differentially private 
learning: \emph{output perturbation}~\citep{chaudhuri2011differentially} using the 
Gaussian mechanism, and (streaming) \emph{DP-SGD}~\citep{abadi2016deep}.
As our results will show, exploiting the specific structure of randomization inside the 
algorithms allows us to establish even tighter guarantees, especially in the high-privacy 
regime that is of most interest in the context of generalization.
The proofs of all results can again be found in Appendix~\ref{app:proofs}.

\subsection{Approximate Max-Information of the Gaussian Mechanism}\label{subsec:maxinf-gauss}

The \emph{Gaussian mechanism}~\citep{dwork2006gaussian} has central importance in differentially private machine learning.
For any function, $\Psi:\X^n\to\R^d$, denote by $\Delta(\Psi):=\sup_{S\sim S'}\|\Psi(S)-\Psi(S')\|_2$ its \emph{sensitivity} for neighboring datasets $S,S'$.
The Gaussian mechanism with noise multiplier $\sigma$ is defined 
as $\M_{\Psi}(S) = \Psi(S) + \sigma \Delta(\Psi) Z$, where $Z\sim\N(0,\text{I}_d)$ 
is standard Gaussian noise, sampled independently from all other quantities.
The privacy properties of the Gaussian mechanism are well studied: 
$\M_\Psi$ is not $\epsilon$-DP for any $\epsilon$, but for 
any $\epsilon,\delta\in(0,1)$ it is $(\epsilon,\delta)$-DP as long as $\sigma \geq \sqrt{2\log(1.25/\delta)}/\epsilon$. A tighter analytic characterization that, 
in particular, also holds for $\epsilon >1$, can be found in \citet{balle2018improving}.
Furthermore, $\M_\Psi$ is $(\rdpOrder, \, \frac{\rdpOrder}{2\sigma^2})$-RDP for every $\rdpOrder>1$~\citep[Corollary~3]{Mironov2017}.

Because of its simplicity and well-understood properties for high-dimensional
problems, the Gaussian mechanism is a common choice for privatizing models 
after they have already been trained, as long as suitably tight bounds 
on the learning algorithms' sensitivity are known. 
The latter is the case, for example, for \emph{regularized 
logistic regression}~\citep{chaudhuri2011differentially},  or 
\emph{principal component analysis}~\citep{chaudhuri2012near}.
The following theorem provides a bound on the approximate max-information of the Gaussian mechanism.

\begin{restatable}{theorem}{thmgauss}
\label{thm:maxinf-gauss-simple}

Let $\Psi:\mathcal{X}^n\to\R^d$ be a function with sensitivity at 
most $\zeta>0$, and let $\M(S) = \Psi(S) + \sigma \zeta Z$ for $Z\sim\mathcal{N}(0,I)$ 
be the \emph{Gaussian mechanism} with noise multiplier $\sigma>0$ applied to $\Psi$.
Let $\nu:=\frac{n}{4\sigma^2}$, and for any $\beta\in(0,1)$ write $L:=\log(1/\beta)$. 
If the elements of $S$ are independent, then the \emph{approximate max-information of $\M$} 
fulfills,
  \begin{align}
    I_{\infty}^{\beta}(\M(S),S) 
    &\leq \nu + 
    \inf_{\lambda\in (0,r_{\nu})}\ \frac{1}{\lambda}\Big[ L  - \log\bigl(1-(\lambda+\lambda^2)\nu\bigr) \Big]
    \label{eq:maxinf_gauss_as_inf}
    \intertext{with $r_{\nu}=\sqrt{\frac{1}{4}+\frac{1}{\nu}}-\frac12$. In more explicit but less tight forms:}
    &\leq 2\nu + \bigl(\onehalf\nu+\sqrt{\onequarter\nu^2+\nu}\bigr)L + \sqrt{2(\nu^2+2\nu)L},
    \label{eq:maxinf_gauss_in_L}
    \intertext{and relaxed further, sorted by powers of $\nu$:}
    &\leq \bigl(2 + \sqrt{2L} + L\bigr)\nu + \bigl(L+2\sqrt{L}\bigr)\sqrt{\nu}.
    \label{eq:maxinf_gauss_in_nu}
  \end{align}
\end{restatable}

\begin{figure}[t]
\centerline{\includegraphics[width=.98\textwidth]{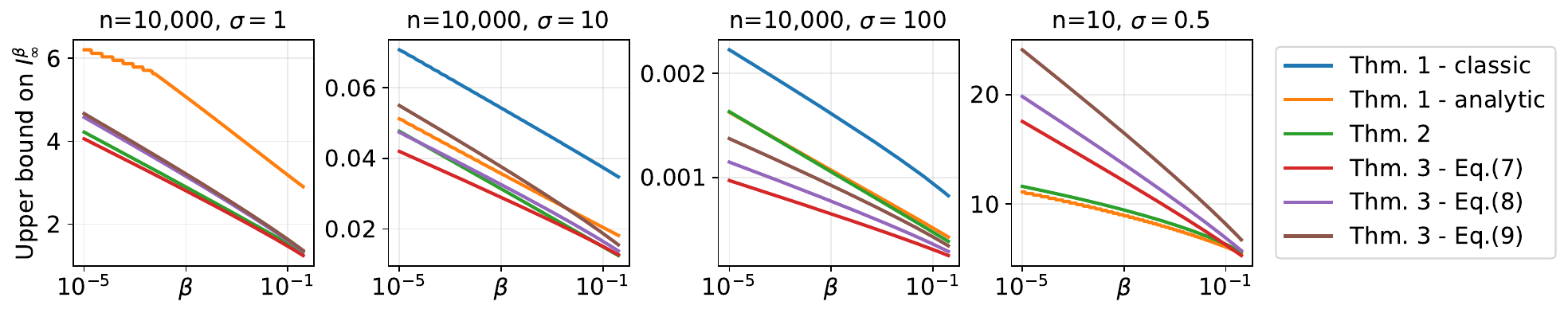}}
\caption{Numeric upper bounds on the approximate max-information of the \emph{Gaussian mechanism} (normalized by $1/n$), as obtained from \Cref{thm:maxinf_from_epsdeltadp}, \Cref{thm:maxinf_from_renyi}, and \Cref{thm:maxinf-gauss-simple} for different values of $n$ and $\sigma$. 
In the left three panels, the explicitly derived bound of \Cref{thm:maxinf-gauss-simple} is always tightest, and the analytic expressions are usually close in value. 
Of other bounds, the bound based on \Renyi-DP is competitive in case of \emph{low privacy} ($\sigma=1$), but less so in the \emph{high privacy} regime ($\sigma=100$).
The bound based on $(\epsilon,\delta)$-DP is competitive with the RDP-bound for 
high privacy settings, at least when evaluated by the analytic relation between
$\sigma$ and $(\epsilon,\delta)$ of~\citep{balle2018improving}. Inferring 
$(\epsilon,\delta)$ from the classic formula generally leads to worse results, 
and is not applicable in the low privacy setting due to its restriction to $\epsilon\in(0,1)$.
The fourth panel shows an case of very low privacy. Here, the $(\epsilon,\delta)$-DP and the \Renyi-DP bounds are preferable over the explicitly derived ones.
}\label{fig:gaussian_mechanism}
\end{figure}

\begin{figure}[t]
  \centerline{\includegraphics[width=.98\textwidth]{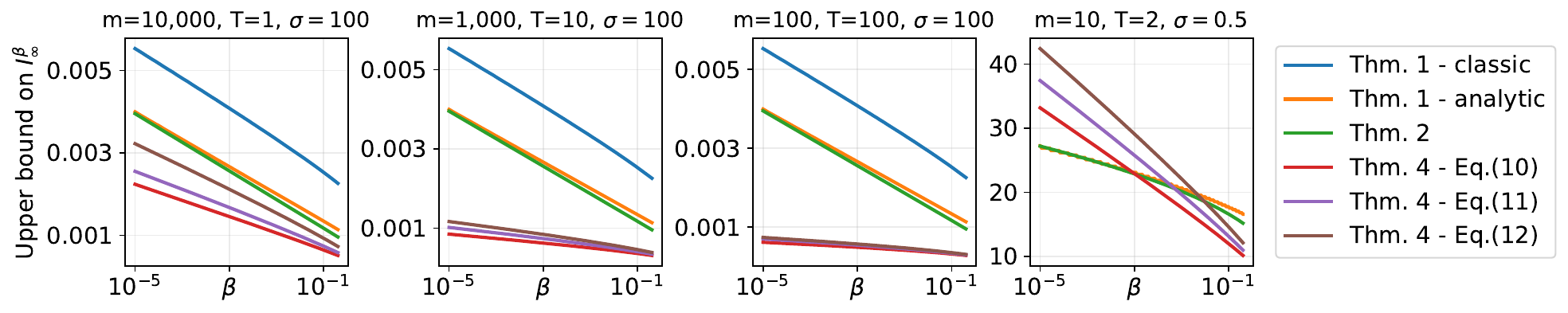}}

\caption{Numeric upper bounds on the approximate max-information of \emph{streaming DP-SGD}
(normalized by $1/n$), as obtained from \Cref{thm:maxinf_from_epsdeltadp}, \Cref{thm:maxinf_from_renyi}, and \Cref{thm:maxinf_dpsgd} for different values of $m$, $T$, and $\sigma$. 
In the left three panels, the three explicit bounds of \Cref{thm:maxinf_dpsgd} are 
always the tightest. Increasing the number of update steps at proportionally reduced
batchsize improves the bounds of \Cref{thm:maxinf_dpsgd}, while the bounds based purely
on privacy properties do not benefit. Of the other bounds, the bounds based on \Renyi-DP 
and $(\epsilon,\delta)$ derived by the analytic rule are comparable and better 
than using classic conversion rule.
The fourth panel shows the limit case of only two update steps with low privacy. 
Here, the \Renyi bound and the $(\epsilon,\delta)$ bound become 
preferable for small values of $\beta$.}\label{fig:dpsgd_mechanism}
\end{figure}

\begin{discussion}
First, we observe that the upper bounds \eqref{eq:maxinf_gauss_as_inf}--\eqref{eq:maxinf_gauss_in_nu} 
are computable explicitly from known quantities: the noise multiplier $\sigma$, the dataset size $n$,
and the desired value of $\beta$. 
Bound \eqref{eq:maxinf_gauss_as_inf} requires a one-dimensional optimization over a finite interval for this, which can easily be performed numerically, especially given that it is one-dimensional and strictly convex in $\lambda$ (Lemma~\ref{lem:proof_inf_convex}).
For theoretical studies, \cref{eq:maxinf_gauss_in_L,eq:maxinf_gauss_in_nu} provide analytical relaxations, in terms of powers of $L$, and of $\nu$ respectively. 
In particular, the latter allows insights into the scaling behavior for practical settings: as $\nu=\frac{n}{4\sigma^2}$,  \cref{eq:maxinf_gauss_in_nu} is of order $O(n/\sigma^2 + \sqrt{n/\sigma^2})$.
This resembles the bounds of the previous section, thinking of $\sigma$ as inversely proportional to the privacy parameter. 
Evaluated numerically, however, the specific bounds of \Cref{thm:maxinf-gauss-simple} are 
typically tighter, see our results in Figure~\ref{fig:gaussian_mechanism}.
\end{discussion}

\begin{wrapfigure}[14]{r}{0.5\textwidth}
\vspace{-2.3\baselineskip} %
\begin{minipage}{0.5\textwidth}
\begin{algorithm}[H]
  \caption{\texttt{DP-SGD}}\label{alg:DPSGD-stream}
  \begin{algorithmic}[1]
    \INPUT input data batches $B_1,\dots,B_T$, clipping threshold $\zeta$, noise multiplier $\sigma$ %
    \REQUIRE subroutine \textsc{ModelUpdate} 
    \smallskip\STATE $\theta_0\leftarrow $ initialize model parameters 
    \FOR{$t = 1,\dots, T$}
    \STATE $z \leftarrow \text{ sample from }\N(0,\text{I}_d)$
    \STATE $\displaystyle u_t \leftarrow \sum\nolimits_{x\in B_t}
    \text{clip}\bigl(\nabla\ell(x,\theta_{t-1}), \zeta\bigr) + \sigma\zeta z$
    \qquad \text{with $\text{clip}(v,\zeta):=\min\{1,\zeta/\|v\|\}v$.}
    \STATE $\displaystyle\theta_{t} \leftarrow \textsc{ModelUpdate}(u_t; \,\theta_{0},\dots,\theta_{t-1})$ %
    \STATE \textbf{yield} $\theta_t$ 
    \ENDFOR
\end{algorithmic}
\end{algorithm}
\end{minipage}
\end{wrapfigure}

\subsection{Approximate Max-Information of Streaming DP-SGD}\label{subsec:maxinf-dpsgd}
In \emph{deep learning}, tight bounds on the overall sensitivity of the learning process are typically not known, so the method of choice is not output perturbation, but some variant of \emph{differentially private stochastic gradient descent (DP-SGD)}~\citep{abadi2016deep}.
DP-SGD performs iterative updates of the model parameters, where in each step a
bound on the sensitivity is enforced by clipping the length of the contributing 
gradients to a maximum length. Subsequently, the Gaussian mechanism is employed
to privatize the update vector.

\Cref{alg:DPSGD-stream} provides pseudo-code, relying on a subroutine \textsc{ModelUpdate} for combining the current (privatized) update vector, $u_t$ with the previous parameters $\theta_1,\dots,\theta_{t-1}$ to produce a new update $\theta_t$.
In classic SGD that is $\theta_t \leftarrow \theta_{t-1} - \eta u_t$ for a 
learning rate $\eta$, but our results are not restricted to this choice. 
Any update rule that depends on the dataset only through the private update vectors can be used,  \eg, SGD with momentum 
and/or weight decay, or Adam~\citep{tang2024DP,hudişteanu2025dpmicroadam},
see Appendix~\ref{app:extendedbackground}. 
For simplicity, we refer to the algorithm as DP-SGD.

The following theorem provides a bound on the approximate max-information 
of DP-SGD in a setting with fixed-size, non-overlapping, batches, as is 
common in decentralized/federated learning or learning from data streams.
\begin{restatable}{theorem}{thmdpsgd}
\label{thm:maxinf_dpsgd}
  Let $\A$ be the \emph{DP-SGD algorithm~\ref{alg:DPSGD-stream}} with hyperparameters $\zeta$ (clipping threshold) and $\sigma $ (noise multiplier), that is run on $T$ data batches of size $m$, which contain independent data samples.
  For any $\beta\in(0,1)$, we write $L=\log(1/\beta)$. Then, the 
  \emph{approximate max-information of $\A$} fulfills 
  \begin{align}
    I^{\beta}_{\infty}(\A(S),S) &\leq 2T\nu + 
    \inf_{\lambda\in(0,r_{4\nu})} \frac{1}{\lambda}\Bigl[L -\frac{T}{2}\log\bigl(1-4(\lambda+\lambda^2)\nu\bigr) \Big],
    \label{eq:maxinf_dpsgd}
\intertext{where $\nu=\frac{m}{4\sigma^2}$, and $r_{4\nu}=\frac12\sqrt{1+\frac{1}{\nu}}-\frac12$.
In less tight, but more explicit, form:}
&\leq 4T\nu  + 2\bigl(\nu+\sqrt{\nu^2+\nu}\bigr)L + 2\sqrt{(4\nu^2+2\nu)TL}.
\label{eq:maxinf_dpsgd_by_L}
\intertext{Relaxed further and sorted by powers of $\nu$, it holds that:}
&\leq 4(T+L+\sqrt{TL})\nu + 2(L+\sqrt{2TL})\sqrt{\nu}.
\end{align}
\end{restatable}

\begin{discussion}
As in the case of the Gaussian mechanism, the bounds of \Cref{thm:maxinf_dpsgd} can 
easily be evaluated numerically or analytically, because they depend only on explicitly 
known quantities, and the optimization problem in~\eqref{eq:maxinf_dpsgd} is one-dimensional and, in fact, strictly convex. 
From the explicit expressions one can see that in the typical regime of $T\geq L$ and $m\geq \sigma^2$, the dominant term is of order 
$\frac{Tm}{\sigma^2}$, which is comparable to the earlier results, as $Tm$ is the total number of data samples used. 

For $T=1$ we can compare directly to \Cref{thm:maxinf-gauss-simple} with a half as large value of $\sigma$ (because a clipping constant $\zeta$ implies sensitivity $2\zeta$), observing that \Cref{thm:maxinf_dpsgd} provides better constants. 
This is a consequence of the fact that we know DP-SGD's feature function is a sum over individual examples rather than an arbitrary function of sensitivity $2\zeta$, so we can exploit stronger concentration results.

Another observation is that all three bounds are decreasing under a transformation 
$(T,m)\mapsto (\alpha T,m/\alpha)$ for any $\alpha>1$, \ie the tightest bound
could be achieved by running with $Tm$ batches of size $1$. 
In practice, however, this would not be advisable, as DP-SGD adds noise 
after each update step, so the algorithm's utility would suffer. 
\end{discussion}

\section{Tight PAC-Bayes Generalization Bounds from DP-SGD Priors}\label{sec:pacbayes}
The bounds on the approximate max-information of the previous sections provide qualitative 
and quantitative insights into the memorization properties of private learning algorithms.
They can also be used to construct tight numerical
bounds on the generalization abilities of trained models, regardless of whether they were trained with privacy or not, which we demonstrate using the PAC-Bayes framework.

\subsection{Background}\label{subsec:pacbayes_background}
PAC-Bayes learning theory allows establishing formal guarantees 
on the \emph{generalization} abilities of a learning method by 
providing high-probability upper bounds on the difference 
between a model's empirical risk (training loss) and its 
true risk (test loss)~\citep{mcallester1999some,maurer2004note}. 

\newcommand{\modelClass}{\mathcal{H}}
\newcommand{\modelVar}{h}
Specifically, let $\modelClass$ be a model class, and let 
$\ell:\X\times\modelClass\to[0,1]$ be a loss function that measures 
the quality of a model, $\modelVar$, on a data point $x$.
For an \iid training dataset, $S$, of size $n$, we denote by  
$\erisk(\modelVar):=\frac1n\sum_{x\in S}\ell(x,\modelVar)$ the \emph{empirical
risk} of $\modelVar$, and by $\risk(\modelVar):=\E\ell(x,\modelVar)$ the \emph{true risk}. 
Then, for any \emph{data-independent} prior, $\pi$ over $\modelClass$ one can prove that for all $\delta\in(0,1)$ it holds with 
probability at least $1-\delta$ that, uniformly over all \emph{posterior} 
distributions $\rho$ over $\modelClass$ : 
\begin{align}
  \kl(\erisk(\rho)\,\|\,\risk(\rho)) \leq \frac{\KL(\rho\|\pi) +
  \log\frac{2\sqrt{n}}{\delta}}{n}, \label{eq:introPACBayes}
\end{align}
where $\erisk(\rho)=\E_{\modelVar\sim\rho}[\erisk(\modelVar)]$ and $\risk(\rho)=\E_{\modelVar\sim\rho}[\risk(\modelVar)]$,
and $\kl(\hat p\|p)=\KL(\text{Ber}(\hat p)\|\text{Ber}(p)) = \hat
p\log\frac{\hat p}{p}+(1-\hat p)\log\frac{1-\hat
p}{1-p}$\footnote{More intuitive notions of similarity between
  $\erisk(\rho)$ and $\risk(\rho)$, \eg on their difference, can be
  derived from~\eqref{eq:introPACBayes} using standard relaxations of the KL divergence. See, \eg, 
  \citet{rivasplata2020pac} for more details. }.
Unfortunately, it is often very difficult to find data-agnostic priors
$\pi$ that allow for posteriors $\rho$ with small $\erisk(\rho)$
and $\KL(\rho\|\pi)$ at the same time. Consequently, the resulting 
generalization guarantees tend to be vacuous in practice, \ie 
numerically they do not improve over the trivial bound $\risk(\rho)\leq 1$.

A way to achieve tighter bounds, specifically to find priors that
enable smaller values of $\KL(\rho\|\pi)$, is to extend~\eqref{eq:introPACBayes}
to allow for \emph{distribution-dependent} or \emph{data-dependent
priors}~\citep{catoni2007pac,lever2013tighter,Parrado2012PAC}.
Specifically, \citet{dziugaite2018data} demonstrated that the prior
distributions can be \emph{learned} on the training data itself, 
if one uses an $\epsilon$-DP algorithm. 
Unfortunately, as mentioned above, $\epsilon$-DP is a quite restrictive 
criterion compared to $(\epsilon,\delta)$-DP with $\delta>0$, and 
so $\epsilon$-DP algorithms are not commonly used for machine learning 
tasks.

\citet{rivasplata2020pac} rephrased the above result to cover all 
algorithms with small approximate max-information, however, without 
providing examples of such algorithms besides $\epsilon$-DP ones. 
The result from~\citet{rogers16maxinfo}
for $(\epsilon,\delta)$-DP algorithms does not suffice in this context,
because it lacks the explicit constants required for numeric evaluation. 
Consequently, the bounds we proved in \Cref{sec:general,sec:maxinf-practical}
are the first candidates for such a procedure.

\subsection{Non-vacuous Generalization Bounds for Deep Networks with DP-SGD Priors\texorpdfstring{\!}{}}
\label{subsec:pacbayes_results} 
We now demonstrate how the max-information framework allows the construction of non-vacuous generalization bounds in a deep learning setting with a completely standard parametrization.
Besides our results from Sections~\ref{sec:general} and~\ref{sec:maxinf-practical}, we 
rely on a variant of \citep[Theorem 4.2]{dziugaite2018data} and ~\citep[Lemma~7]{rivasplata2020pac}, 
which we state here for reference in a form compatible with the rest of our notation.

\begin{theorem}\label{thm:pacbayes_with_maxinf}
  For a dataset $S=(X_1,\dots,X_n)$ of \iid data points, let $\pi$ be 
  a prior distribution over models, obtained from the output of running 
  an algorithm $\A$ on $S$.
  Then, for any $\delta\in(0,1)$, with probability at least
  $1-\delta$ over the sampling of $S$ and randomness of $\A$, that uniformly for all 
  posterior distributions $\rho$, we have:
\begin{align}
    \kl\bigl(\erisk(\rho) \| \risk(\rho)\bigr) &\leq
    \frac{\KL(\rho\|\pi) + I^{\delta/2}_{\infty}(\mathcal{A}(S),S) + \log(\frac{4\sqrt{n}}{\delta})}{n}.
    \label{eq:PACBayes-klbound}
\end{align}
\end{theorem}

As an immediate consequence of \Cref{thm:pacbayes_with_maxinf} (by setting $\rho=\pi$),
we obtain a simpler generalization bound that is particularly useful for privately 
trained models:

\begin{corollary}\label{cor:pacbayes_with_maxinf_for_prior}
  For a dataset $S=(X_1,\dots,X_n)$ of \iid data points, let $\rho$ be 
  a distribution over models, obtained from the output of running an 
  algorithm $\A$ on $S$.
  Then, for any $\delta\in(0,1)$, with probability at least
  $1-\delta$ over the sampling of $S$ and randomness of $\A$, we have:
\begin{align}
    \kl\bigl(\erisk(\rho) \| \risk(\rho)\bigr) &\leq
    \frac{I^{\delta/2}_{\infty}(\mathcal{A}(S),S) + \log(\frac{4\sqrt{n}}{\delta})}{n}.
    \label{eq:DPSGD-klbound}
\end{align}
\end{corollary}

\begin{discussion}
The generalization bounds hold for any mechanism of deriving distributions over 
the model set from the algorithms' output. The most common choice is to use 
isotropic Gaussian distributions centered at the learned parameter vectors, 
which makes the $\KL$ term appearing in the bounds explicitly computable. 
For other possibilities, see Appendix~\ref{app:extendedbackground}.
In the subsequent section we show that the resulting bounds can be 
numerically non-vacuous even for deep networks in their standard weight-matrix 
parametrization, \ie without having to rely on discretized, low-rank or compressed parametrizations that several earlier works employed~\citep{arora2018stronger,lotfi2022PACBayes,zakerinia2025low} or reserving a subset of the training data just for learning a prior. 
\end{discussion}

\begin{table}[t]
\centering
\caption{Numeric upper bounds, $\mathcal{B}$, on the risk, $\risk$, for overparameterized deep 
networks on MNIST, obtained by numerically inverting  \eqref{eq:DPSGD-klbound}
for the DP models $h_0$, and \eqref{eq:PACBayes-klbound}  for $\rho$. 
Reported results and standard deviations are from three runs with different 
random seeds.}\label{tab:experiments_mnist}
\scalebox{0.88}{
\begin{tabular}{ccccccc}
\toprule
Method 
& $\erisk(h_0)$ 
& $\erisk(\rho)$ 
& $\kappa$ 
& Complexity/n 
& $\mathcal{B}(h_0)$ 
& $\mathcal{B}(\rho)$ 
\\
\midrule
data-indep. 
& -- 
& $0.104_{\pm 0.003}$ 
& -- 
& $0.136_{\pm 0.003}$ 
& -- 
& $0.325_{\pm 0.006}$ 
\\
DP-SGD (streaming)
& $0.238_{\pm 0.019}$ 
& $0.163_{\pm 0.002}$ 
& $1381$ 
& $0.034_{\pm 0.001}$  %
& $0.337_{\pm 0.021}$ 
& $0.273_{\pm 0.004}$ %
\\
DP-SGD (Poisson subs.)
& $0.164_{\pm 0.001}$ 
& $0.144_{\pm 0.001}$ 
& $1224$ 
& $0.025_{\pm 0.001}$ 
& $0.248_{\pm 0.001}$ 
& $0.233_{\pm 0.001}$ \\
\bottomrule
\end{tabular}}
\end{table}

\subsection{Numerical Experiments} \label{subsec:experiments}
In this section, we illustrate the behavior of our results in \Cref{sec:general,sec:maxinf-practical}
when used in combination with \Cref{thm:pacbayes_with_maxinf} or \Cref{cor:pacbayes_with_maxinf_for_prior}.
We use the setup of \citep{garcia2025some, perez2021tighter}, \ie optimizing a PAC-Bayes bound over a family of Gaussian posteriors given a Gaussian prior. 
For details of the experimental setup, see Appendix~\ref{app:experiments}.

For MNIST, we use an MLP model with $1.2M$ parameters, and compare three different choices of prior: 1) a data-independent Gaussian prior as in~\citep{garcia2025some}. 2) a Gaussian prior whose mean is learned by streaming DP-SGD and its max-information is computed based on \Cref{thm:maxinf_dpsgd}. 3) a Gaussian prior which its mean is learned by standard multi-epoch DP-SGD with Poisson subsampling, which has better privacy guarantees due to amplification \citep{Zhu2019Poisson}, and its max-information is computed based on \Cref{thm:maxinf_from_renyi}.

\Cref{tab:experiments_mnist} shows that both private priors achieve tighter bounds than the data-independent prior. Additionally, the comparison between the prior learned from streaming DP-SGD and from DP-SGD with Poisson subsampling highlights the importance of our general results, especially \Cref{thm:maxinf_from_renyi}: Even though in the streaming setting, \Cref{thm:maxinf_dpsgd} is tighter than the bound from \Cref{thm:maxinf_from_renyi}, when we have full access to the dataset, Poisson subsampling can achieve stronger privacy guarantees, resulting in smaller max-information bounds. Moreover, it allows training for multiple epochs, which is often necessary to obtain good models.

For CIFAR-10, we use three CNN models with approximately $10K$, $5.9M$, and $13.2M$ parameters, and compare the approach with a data-independent Gaussian prior with a prior learned by DP-SGD with Poisson subsampling as described above. 

To achieve non-vacuous bounds for large networks, existing approaches either use a compact parameterization to reduce the number of parameters and compute a bound for the compressed network, or use data-splitting, \ie use part of the samples to train a prior and use the remaining to optimize a PAC-Bayes bound.
Although data splitting yields valid bounds, the process that produced the
prior is not analyzed, so the bound does not explain why the prior itself does not
overfit its training data~\citep{lotfi2022PACBayes,garcia2025some}. Additionally,
they posteriors remain close to their learned priors in deeper
architectures~\citep{perez2021tighter}, so the certificate effectively reduces to a held-out evaluation of the prior~\citep{lotfi2022PACBayes}. Our goal is instead to certify the full learning process: we learn the prior from the same samples that enter the bound and account for this dependence through its max-information, without any held-out data. The closest work to ours is \cite{dziugaite2018data}, which applies the max-information bound for $\epsilon$-DP to sampling from the Gibbs posterior to learn the prior for MNIST. However, in practice they approximate the Gibbs posterior by SGLD, which the guarantee does not formally cover. To our knowledge, ours is the first approach that simultaneously applies to models with up to $13M$ trainable parameters, in their original parameter space, avoids using data splitting, and computes max-information terms for the algorithm as executed.

In \Cref{tab:experiments_cifar_h0} we provide the guarantees for the model output of DP-SGD, where the bounds for all parameter sizes are non-vacuous. To the best of our knowledge, these are the first non-vacuous bounds for models of this size without compression or data splitting, for an algorithm used in practice (here, DP-SGD with Poisson subsampling). Additionally, in \Cref{tab:experiments_cifar_rho} we show that further optimizing the posterior can improve the bounds, and similarly to MNIST, achieve better results compared to using data-independent priors, which cannot provide non-vacuous bounds or bigger networks.

\begin{table}[t]
\centering
\caption{Empirical risk and numeric upper bounds for DP models, $h_0$, on CIFAR-10, with max-information $\kappa$.
Reported results and standard deviations are over three runs with different
random seeds.}
\label{tab:experiments_cifar_h0}
\begin{tabular}{ccccc}
\toprule
\# Layers & \# Parameters & $\kappa$
& $\erisk(h_0)$ & $\mathcal{B}(h_0)$ \\
\midrule

4
& 9,946
& 757
& $0.705_{\pm 0.021}$
& $0.780_{\pm 0.019}$ \\

9
& 5,851,338
& 342
& $0.748_{\pm 0.019}$
& $0.797_{\pm 0.017}$ \\

15
& 13,193,930
& 757
& $0.782_{\pm 0.011}$
& $0.848_{\pm 0.009}$ \\

\bottomrule
\end{tabular}
\end{table}

\begin{table}[t]
\centering
\caption{Empirical risk $\erisk(\rho)$, max-information $\kappa$, total complexity over $n$ , and numeric upper bounds $\mathcal{B}$, for three ConvNet models of different sizes trained on CIFAR-10. Reported results and standard deviations are over three runs with different random seeds.}
\label{tab:experiments_cifar_rho}
\begin{tabular}{ccccccc}
\toprule
\# Layers & \# Parameters & Method
& $\erisk(\rho)$ & $\kappa$ & Complexity$/n$ & $\mathcal{B}(\rho)$ \\
\midrule

\multirow{2}{*}{4}
& \multirow{2}{*}{9,946}
& data-indep.
& $0.591_{\pm 0.012}$
& --
& $0.051_{\pm 0.004}$
& $0.738_{\pm 0.007}$ \\
&& DP-SGD
& $0.534_{\pm 0.011}$
& $757$
& $0.074_{\pm 0.004}$
& $0.716_{\pm 0.007}$ \\

\midrule
\multirow{2}{*}{9}
& \multirow{2}{*}{5,851,338}
& data-indep.
& --
& --
& --
& \text{vacuous} \\ %
&& DP-SGD
& $0.629_{\pm 0.002}$
& $342$
& $0.061_{\pm 0.004}$
& $0.782_{\pm 0.003}$ \\

\midrule
\multirow{2}{*}{15}
& \multirow{2}{*}{13,193,930}
& data-indep.
& --
& --
& --
&  \text{vacuous} \\ %
&& DP-SGD
& $0.648_{\pm 0.007}$
& $757$
& $0.029_{\pm 0.003}$
& $0.757_{\pm 0.012}$ \\

\bottomrule
\end{tabular}
\end{table}

\section{Conclusion}
In this work, we established the first explicit upper bounds on the $\beta$-approximate 
max-information of differentially private learning algorithms in regimes relevant to 
modern deep learning.
Our results provide a principled connection between differentially private deep
learning and PAC-Bayes generalization theory, advancing over previous analyses
that were restricted to pure differential privacy or restricted problem classes.
Leveraging our bounds, we showed that DP-SGD can be used to learn data-dependent 
priors while retaining generalization guarantees expressible directly in terms 
of optimization hyperparameters.

Despite these achievements, a number of limitations remain that we see as 
opportunities for future work. 
On the technical level, our explicit bounds for DP-SGD were restricted 
to a setup with fixed-size non-overlapping batches. In real-world settings, 
DP-SGD can often be run as a batch algorithm, in particular data can be accessed 
in a random-access fashion. 
In that case, sampling the data in batches that are random and potentially 
overlapping is known to improve the privacy guarantees, \eg by 
employing \emph{privacy amplification}~\citep{ponomareva2023dpfy}.
Even in the streaming setting, tighter guarantees might be possible, 
\eg through the use of noise that is correlated between update steps 
rather than independent~\citep{pillutla2025correlatednoisemechanismsdifferentially}.
While such techniques are already covered in our analysis through 
the general-purpose results of Section~\ref{sec:general} (and in our experiments we do use, \eg, Poisson batch sampling), it would be interesting to 
see how far the explicit analysis of Section~\ref{sec:maxinf-practical} 
could be extended to these techniques. In particular, the role of 
batch size versus number of update steps would be an interesting candidate 
to explore in this context.

On a conceptual level, we would like to emphasize that even though 
our results apply to practically relevant setups, they should not be 
misinterpreted as a guide for practitioners. 
In particular, generalization bounds like the one we present are 
designed to hold uniformly across all data distributions and 
mechanisms for determining a model from the training data. As a
consequence, they are necessarily conservative and often loose. 
We rather think of our results as a source of insights into how the core 
parameters of privacy and generalization interact. For example, the 
linear scaling of our main results with respect to the dataset size, 
in combination with the $I^{\delta/2}_{\infty}/n$ correction term in 
the PAC-Bayes bounds suggests that---at least by this analysis---the 
most promising strategy to guarantee generalization is not simply to 
use more data, but to enforce stronger privacy. 
Of course, the former is not contradictory to the latter, as more data 
can allow for larger noise multipliers without having to sacrifice 
utility, but we leave a formal study of these aspects to future work.

\subsection*{AI Use Statement}

In this work, we used generative AI tools for suggesting derivation paths 
for some of the mathematical statements, as well as for double-checking our 
proofs, in particular making sure that constants are tracked correctly. 
We have not used AI tools to formulate research questions or interpret results.
Additionally, we used generative AI tools to identify relevant related 
work, suggest edits to the manuscript for better readability, and implement 
routine tasks, such as numerically evaluating bounds and data plotting.
All AI-assisted work was reviewed manually for correctness and readability.
We take responsibility for the final content of this work, including any
parts produced by or with the aid of generative AI.

{\small
  \bibliography{references}
  \bibliographystyle{conference}
}

\clearpage
\appendix
\onecolumn
\section{Extended Background}\label{app:extendedbackground}
In this section we provide additional background information on the 
concepts introduced in Section~\ref{sec:background}. 

\subsection{Learning Setup}
The learning setup in which we work (input space $\X$, model space $\Y$, loss function $\ell:\X\times\Y\to\R_+$) provides a unified notation for a number of standard learning scenarios.
Most commonly, in supervised classification, $\X$ would contain input-output pairs, 
$\mathcal{Y}$, could be the model parameters, \eg of a neural network, and $\ell$ 
measures if the model predicts the correct class label or not. 
However, the notation also applies more widely. For example, in unsupervised learning, 
the input set would be unlabeled data, the output could be a set of cluster centroids, 
and the loss function measures the distance between a data and its closest cluster 
centroid.

\subsection{Differentially-Private Stochastic Gradient Descent (DP-SGD)}
Several variants and improvements of DP-SGD and its privacy analysis have been developed. 

One aspect under active development is \emph{batch generation}: it has 
been shown that when data batches are created randomly, \eg by 
Poisson sampling~\citep{kasiviswanathan2011what}, \emph{balls-and-bins}~\citep{chua2025balls}, $k$-out-of-$t$ allocation~\citep{feldman2025privacy}, or even 
just \emph{data shuffling}~\citep{erlingsson2019amplification}, 
stronger privacy guarantees can be proven than what would be 
provided by analyzing DP-SGD simply as a concatenation of 
Gaussian mechanisms. 
However, these techniques tend to result in batches with 
different (random) batch sizes, which creates problems 
with practical implementation and resource utilization. 
A promising alternative is \emph{Truncated Poisson Sampling}~\citep{chua2024scalable},
which enforces an upper bound on the batch size and has been used,
\eg, in training Google's VaultGemma LLM~\citep{sinha2025vaultgemmadifferentiallyprivategemma}.

In orthogonal efforts, it has been shown that the use of correlated 
instead of uncorrelated noise, equally strong privacy guarantees 
can often be achieved with less overall added noise and higher 
model utility~\citep{pillutla2025correlatednoisemechanismsdifferentially}.

We believe that an extension of our current analysis, which assumes 
classic fixed-size disjoint batches and independent noise, to 
these techniques should be possible, but would require changes to 
our proof techniques that go beyond the scope of this work.

Our discussion of DP-SGD left the choice of the \emph{ModelUpdate} 
routine mostly open, because the specific form of the update does not
affect either the max-information or the privacy guarantees, since 
the analysis depends only on the privatized update vectors. 

Classically, the update steps of (DP-)SGD, given an update vector of $u_t$ is 
\begin{itemize}
\item $\theta_t \leftarrow \theta_{t-1} - \eta_t u_t$ 
\end{itemize}
for some (constant or time-varying) learning rate $\eta_t$~\citep{abadi2016deep}. 

Including \emph{momentum} of strength $\beta$ and \emph{weight decay} 
of strength $\alpha$, changes the rule to 
\begin{itemize}
    \item $\theta_t \leftarrow (1-\alpha)\theta_{t-1} - \eta_t m_t$, where
    \item $m_t \leftarrow \beta m_{t-1} + (1-\beta)u_t$, 
\end{itemize}
is the momentum term that is also updated in each step. 

Finally, DP-Adam~\citep{anil2022large} or variants~\citep{tang2024DP} 
additionally perform a component-wise normalization of the updates, \eg 
\begin{itemize}
    \item $\theta_t = \theta_{t-1} - \frac{\eta_t}{\sqrt{\hat{v}_t} + \epsilon} \hat{n}_t$, for 
\item $\hat{n}_t = \frac{m_t}{1 - \beta_1^t}$ with $m_t = \beta_1 m_{t-1} + (1 - \beta_1) u_t$, and
\item $\hat{v}_t = \frac{v_t}{1 - \beta_2^t}$ with $v_t = \beta_2 v_{t-1} + (1 - \beta_2) u_t^2$.
\end{itemize}

One can see that in all of them, the model parameters are a deterministic
linear (DP-SGD) or non-linear (DP-Adam) combination of the update vectors,
so all of them are covered by our analysis.

\subsection{PAC-Bayes generalization bounds.}\label{app:pacbayes}

\paragraph{Construction of priors and posteriors.}

At the core of PAC-Bayes theory lie stochastic classifiers, \ie 
distributions, $\pi$ (prior) and $\rho$ (posterior) over the 
model space from which individual models are sampled.
Our work is agnostic to how exactly the distributions are constructed
from the output of a learning algorithm, as this step constitutes
a postprocessing operations and therefore cannot lead to an increase
of the max-information (or a reduction of privacy).
In the literature, a number of suitable options exist:
\begin{itemize}
    \item Most classically, the training step outputs the parameter 
    vector, $\theta$, of one model, from which one constructs a
    Gaussian distribution centered at $y$ with fixed
    (often isotropic) covariance matrix. The variance can be treated
    as a hyper-parameter and chosen by model-selection. 
    The use of a Gaussian form is motivated by the fact that it leads
    to explicit expression for $\KL(\rho\|\pi)$, which can then serve 
    as part of the objective function for the model training step~\citep{ambroladze2006tighter,dziugaite2017computing}.  
    \item In \emph{Bayesian Deep Learning}, the model parameters directly
    correspond to the parameters of a distribution over models, \eg by 
    learning not only the mean but also (co)variance term of a Gaussian~\citep{blundell2015weight,rivasplata2019PAC}. 
    \item When the space of model is finite, \eg because the model 
    coefficients are quantized, one can choose the posterior as a 
    $\delta$-peak at the (quantized) posterior. 
    This way, the bound applies also to deterministic classifiers~\citep{zhou2018non,lotfi2022PACBayes}.
    \item Exploiting the iterative nature of (DP-)SGD, one can the 
    construct the mean and covariance matrix of a Gaussian distribution 
    from its iterates~\citep{maddox2019simple}, or treat them as 
    approximate samples from the \emph{Gibbs posterior}~\citep{mandt2017stochastic,maurer2025generalization}.  
\end{itemize}

\paragraph{From \emph{kl} to more interpretable generalization bounds.}
PAC-Bayes generalization bounds of a form $\kl(\erisk\|\risk)\leq C/n$ are
also known as \emph{fast-rate bounds}, because their convergence speed can 
be faster than the classical rate $O(\sqrt{1/n})$, specifically when the 
training risk, $\erisk$, is small~\citep{seeger2002pac,catoni2007pac}.
This comes at the disadvantage that the bound on $\risk$ is only implicit, and the 
bounding quantity, $C/n$, lacks immediate interpretability. 
Two main routes exist to solve this issue.

First, observe that $\kl(q\|p)$ is strictly monotonically 
increasing in $p\geq q$, so it is invertible in that range. 
Consequently, from the implicit form $\kl(\erisk \| \risk) \le  C/n$, 
we can compute an upper-bound on $\risk$ as 
$R \leq \kl^{-1}(\erisk \| C/n)$, where 
\begin{equation}
\kl^{-1}(q\|b) = \text{sup} \bigl\{p\in[q,1] : \kl(q\|p) \le b\bigr\}.
\end{equation}
While $\kl^{-1}$ does not have a simple analytic form, we can compute
it efficiently with numeric techniques, \eg, using a binary search 
or Newton's method.

Alternatively, one can use Pinsker's inequality, $\kl(q\|p)\geq 2(q-p)^2$, 
to derive bounds in the classic \emph{generalization-gap} form, 
$\risk \leq \erisk + \sqrt{\frac{C}{2n}}$.
Due to its analytic expression, this form can also be used, \eg, 
as part of the training objective inside of a learning algorithm.

\newpage
\clearpage
\section{Proofs}\label{app:proofs}

\newcommand{\infoDens}{\imath}
\newcommand{\info}{\infoDens(S,Y)}
\newcommand{\Hell}{H}

In the following sections we provide the proofs that were left out
of the main body of the manuscript due to space limitations. We
present the proofs in a different order to the theorems' appearance
in the text, to provide a gradual increase in complexity and make the
similarities between proofs clearer.

While the proofs do share structural similarities, there are
sufficiently many differences that we present each proof as a
self-contained unit, with just the same notation running through.

As a guide, we provide the following structural outline that all of the
proofs that follow:

\begin{enumerate}
  \item The desired approximate max-information bound can be shown
    using \cite[Lemma 18]{dwork2015generalization} or a generalization thereof and a tail bound
    on the information density $\iota(Y; S)$, which is defined as the
    log-ratio between the joint probability of an input-output pair $(S,Y)$, and the product of their marginal probabilities: 
    \begin{equation}
      \infoDens(Y; S) :=  \log  \frac{p(Y,S)}{p(Y)p(S)}. %
    \end{equation}
  \item Decompose $\infoDens(Y; S)$ (or an upper bound on it) into a sum, over data points/ batch
    entries, of increments $\Lambda_i$ adapted to a filtration
    $\mathcal{F}_{i-1}$ of the preceding steps.
  \item Bound the moment generating function $\log
    \mathbb{E}[\exp(\lambda \Lambda_i) \mid \mathcal{F}_{i-1}] $ of
    each increment (conditional on the history) using privacy
    or the mechanism's explicit form. Some of these bounds are stated in terms of previous points. 
  \item Compose these bounds, starting with the last point and
    working backwards to the first, resulting in a bound on the
    moment generating function  $\log \mathbb{E}[\exp(\lambda
    \infoDens(Y;S))]$ of the information density. Where the
    per-step bounds depend on the preceding points,  we instead we bound the moment
    generating function of the information density minus the sum
    of the per-step bounds, and separately show that this sum is
    small with high probability.
  \item Optimize the order $\lambda$ in this bound, wherever it is valid.
\end{enumerate}

\subsection{Proof of Theorem~\ref{thm:maxinf-gauss-simple}}

We now provide the proof of \Cref{thm:maxinf-gauss-simple}, which we restate here for completeness.

\thmgauss*

\begin{proof}
Without loss of generality, we can assume $\zeta=1/\sigma$ and
$\E[\Psi(S)]=0$, because the max-information of an algorithm
is not affected by a deterministic rescaling or shift of 
its outputs.

By~\citet[Lemma 18]{dwork2015generalization}, adjusted to the fact
the we use natural instead of binary logarithms, in order to establish
$I^{\beta}_{\infty}(\A(S),S)\leq \kappa$, it suffices to prove
$\Prob\{ \imath(\A(S);S)\geq \kappa\}\leq \beta$ where
\begin{align}
  \imath(Y;S) &:= \log \frac{p(Y,S)}{p(Y)p(S)} = \log \frac{p(Y| S)}{p(Y)}.
\end{align}
is called the \emph{information density}.

We do so in three steps: first, we establish a computable upper bound
on $\imath(Y;S)$
(Lemma~\ref{lem:proof_gTZ}). Then, we will derive an upper bound on the
moment-generating function of this upper bound, and afterwards relate it to
tail bounds of the original information density.

\begin{lemma}\label{lem:proof_gTZ}
  \begin{align}
    \Pr\Bigl\{ \imath(Y;S) \geq \kappa \Bigr\} &\leq \Pr\Bigl\{
    g(S,Z) \geq \kappa - \frac{1}{2}\E\|\Psi(S)\|^2 \Bigr\}
    \intertext{for}
    g(S,Z) &:= \bigl\langle \Psi(S), Z\bigr\rangle + \frac{1}{2}\|\Psi(S)\|^2
  \end{align}
  where $Z$ is a standard multivariate Gaussian, independent of all
  other quantities.
\end{lemma}

\begin{proof}
  For any $S$ and $\A$, where the respective expressions are defined,
  it holds almost surely
  for any output $y$:
  \begin{align}
    \imath(\A(S)=y;S) &= \log \frac{p(\A(S)=y|
    S)}{\E_{S'}p(\A(S')=y)} \leq \E_{S'}\log \frac{p(\A(S)=y|S)}{p(\A(S')=y|S')}
    \intertext{by Jensen's inequality, where $S'$ is an independent copy of $S$.
      Inserting $\A(S)=\Psi(S) + Z$, for $Z\sim\mathcal{N}(0,\text{I}_d)$, one
      has $p(\A(S)=y|S)=(2\pi)^{-d/2}e^{-\frac12\|\Psi(S)-y\|^2}$
    and $p(\A(S')=y|S')=(2\pi)^{-d/2}e^{-\frac12\|\Psi(S')-y\|^2}$, \ie}
    &= \E_{S'}\Bigl[ -\frac12\|\Psi(S)\|^2 + \bigl\langle \Psi(S), y\bigr\rangle
    +\frac12\|\Psi(S')\|^2 - \bigl\langle \Psi(S'), y\bigr\rangle \Big]
    \\
    &= \E_{S'}\Bigl[ \bigl\langle \Psi(S)-\Psi(S'), y-
    \Psi(S)\bigr\rangle + \frac12\|\Psi(S)-\Psi(S')\|^2 \Bigr].
    \intertext{Observing that  $\E_{S'}\Psi(S')=\E_S\Psi(S)$ and
      $\E_{S'}\|\Psi(S)-\Psi(S')\|^2= \|\Psi(S)\|^2 +
    \E_{S}\|\Psi(S)\|^2$, it follows that}
    &= \Bigl[  \bigl\langle \Psi(S), y- \Psi(S)\bigr\rangle  +
      \frac12\|\Psi(S)\|^2
    + \frac12\E_S\|\Psi(S)\|^2  \Bigr]
    \intertext{Now, inserting $Y$ for $y$ and observing that
    $Y-\Psi(S)=Z$, we obtain that, almost surely}
    \imath(\Psi(S)+Z;S) &\leq \bigl\langle \Psi(S), Z\bigr\rangle +
    \frac{1}{2}\|\Psi(S)\|^2 + \frac{1}{2}\E_S\|\Psi(S)\|^2 = g(S,Z)
    + \frac{1}{2}\E\|\Psi(S)\|^2
  \end{align}
  which implies the statement of the Lemma.
\end{proof} %

We will also use the following classical property of the Gaussian distribution:
\begin{lemma}\label{lem:gauss_mgf}
  Let $Z$ be a multi-variate standard Gaussian, $Z \sim \N(0, \text{I}_d)$.
  Then, for any fixed vector-valued $a$ and scalar $b$, and it holds
  for every $t\in\R$
  \begin{align}
    \log\E[e^{t\langle Z,a\rangle +b} ] = \log \E[ e^{\frac12 t^2\|a\|^2 + b}]
  \end{align}
\end{lemma}

\begin{proof}
  The proof is elementary, because the law of $\langle Z,a\rangle$ is
  one-dimensional Gaussian with variance parameter $\|a\|^2$, and $b$
  is not affected by the expectation.
\end{proof}

Next, we prove some useful lemmas which could be of independent
interest, which we have not found in these precise forms in the literature.

\begin{lemma}\label{lem:proof_tilted_mgf}
  Let $V$ be a centered vector-valued random variable whose support
  has diameter at most $c$. Then it holds for any vector $x$ and any
  $0\leq t<\frac{2}{c^2}$
  \begin{equation}
    \log\E e^{t\|x+V\|^2} \leq \frac{1}{2}c^2t +
    \frac{t}{1-\frac{1}{2}c^2t}\|x\|^2.
  \end{equation}
\end{lemma}

\begin{proof}
  Denote the law of $V$ by $P$. For any $x$ and $t$, define the
  tilted probability
  measure $Q$ as $\frac{dQ}{dP} = \frac{e^{t\|x+V\|_2^2}}{\mathcal{Z}_{x,t}}$
  with normalization constant $\mathcal{Z}_{x,t} = \E_{P}[e^{t\|x+V\|_2^2}]$.

  By the definition of the Kullback-Leibler divergence and of $Q$, we have
  \begin{equation}
    \log \E_{P}[e^{t\|x+V\|^2}] = t\E_Q[\|x+V\|^2] - \KL(Q\|P)
    \label{eq:proof_logEP_as_EQ_and_KL}
  \end{equation}
  Next, we derive an upper bound on the right hand side. %
  First, observe that $V$ has the same bounded support of diameter at
  most $c$ under $Q$
  as under $P$. Consequently,
  \begin{equation}
    \E_Q\|V-\E_Q[V]\|^2 = \frac12\E_{Q\otimes Q}\|V-V'\|^2 \leq \frac{c^2}{2},
    \label{eq:proof_var_V_under_Q}
  \end{equation}
  and therefore, it holds for any $x$:
  \begin{align}
    \E_Q\|x+V\|^2 &\leq \E_Q\|V-\E_Q[V]\|^2 + \|x+\E_Q[V]\|^2 \leq
    \|x+\E_Q[V]\|^2 + \frac{c^2}{2}.
    \label{eq:proof_squared_norm_mgf_shift}
  \end{align}

  Now, by applying the standard change-of-measure inequality
  $\E_{v\sim Q}[f(v)] \leq \KL(Q\|P) + \log \E_{v\sim P} e^{f(v)}$
  with $f(v)=\langle u,v\rangle$ for some vector $u$, we obtain
  \begin{align}
    \KL(Q\|P) &\geq \langle u,\E_Q[V]\rangle - \log \E_P e^{\langle u,V\rangle}
    \geq \langle u,\E_Q[V]\rangle - \frac{c^2\|u\|^2}{8}
    \intertext{where the second inequality is due to Hoeffding's
      lemma, observing that the range
      of $\langle u,V\rangle$ has length at most $c\|u\|_2$, so
      $\log\E_{P}e^{\langle u,V\rangle} \leq \frac{c^2\|u\|^2}{8}$.
    Taking supremum over $u$, it follows:}
    \KL(Q\|P) &\geq \sup_u \Bigl(\langle u,\E_Q[V]\rangle -
    \frac{c^2\|u\|^2}{8}\Bigr)
    = \frac{2\|E_Q[V]\|^2}{c^2}.
    \label{eq:proof_squared_norm_mgf_KL}
  \end{align}

  Inserting \eqref{eq:proof_squared_norm_mgf_shift} and
  \eqref{eq:proof_squared_norm_mgf_KL}
  into \eqref{eq:proof_logEP_as_EQ_and_KL}, it follows
  \begin{align}
    \log \E_P e^{t\|x+V\|^2} &\leq  \frac{c^2}{2}t + t\|x+\E_Q[V]\|^2
    - \frac{2\|E_Q[V]\|^2}{c^2}.
    \intertext{Expanding the norms and completing the square, it follows}
    &= \frac{c^2}{2}t + \frac{t}{1-\frac{c^2}{2}t}\|x\|^2 -
    (\frac{2}{c^2}-t)\Biggl\|\E_Q[V] - \frac{t}{\frac{2}{c^2}-t}x\Biggr\|^2
    \\
    &\leq \frac{c^2}{2}t + \frac{t}{1-\frac{c^2}{2}t}\|x\|^2,
  \end{align}
  because $t<\frac{2}{c^2}$ by assumption. This completes the proof
  of \Cref{lem:proof_tilted_mgf}.
\end{proof}

\begin{proposition}\label{prop:logmgf-squarednorm}
  Let $f:\X^n\to\mathbb{R}^d$ be a function that fulfills a bounded
  difference inequality with constant $c>0$:
  \begin{align}
    \forall i\in\{1,\dots,n\}:
    \sup_{x_1,\dots,x_n,x'\in\X}\Bigl\|f(x_1,\dots,x_n) &-
    f(x_1,\dots,x_{i-1},x',x_{i+1},\dots,x_n)\Bigr\|\leq c.
    \intertext{Let $X_1,\dots,X_n$ be independent random variables
      $X_i\in\X$ and
    write $X=(X_1,\dots,X_n)$. Then, it holds that:}
    \E\|f(X)-\E[f(X)]\|^2 &\leq \frac12 nc^2.
    \label{eq:squared_norm_variance}
    \intertext{For all $0\leq\lambda<\frac{2}{nc^2}$,}
    \log\E e^{\lambda\|f(X)-\E[f(X)]\|^2} &\leq -\log(1-\frac12
    nc^2\lambda). \label{eq:squared_norm_logmgf}
  \end{align}
\end{proposition}

\begin{proof}
  To prove the statements, we work iteratively over the steps of a
  Doob martingale,
  exploiting exponential tilts from \Cref{lem:proof_tilted_mgf} along the way.
  Write $F:=f(X)-\E[f(X)]$ and introduce the Doob martingale:
  \begin{align}
    M_i := \E[F | \F_i], \qquad D_i := M_i - M_{i-1},
  \end{align}
  where $\F_i=\sigma(X_1,\dots,X_i)$. By construction, $M_0=0$, $M_n=F$,
  and $\E[D_i | \F_{i-1}] = 0$.
  Conditioned on any fixed past, $X_1,\dots,X_{i-1}$, set
  \begin{align}
    g_i(x) &:=
    \E_{X_{i+1},\dots,X_{n}}\bigl[f(X_1,\dots,X_{i-1},x,X_{i+1},\dots,X_n)\bigr].
  \end{align}
  By the bounded difference inequality of $f$, it follows that
  $\|g_i(x)-g_i(x')\|\leq c$ almost surely
  for all $x,x'\in\X_i$, so the conditional law of
  $D_i=g_i(X_i) - \E[g_i(X_i)|\F_{i-1}]$ has diameter at most $c$.
  From \eqref{eq:proof_var_V_under_Q}, now under the conditional law, it follows
that $\E\|D_i\|^2 | \F_{i-1}] \leq \frac{c^2}{2}$. Consequently, by
the telescoping property of $D_i$ and
the orthogonality property of martingales,
\begin{align}
  \E\|f(X)-\E[f(X)]\|^2 &= \E\|F\|^2 = \sum_{i=1}^n \E\|D_i\|^2 \leq
  \frac{nc^2}{2},
\end{align}
which proves \eqref{eq:squared_norm_variance}.

Conditioned on $\F_{i-1}$, all $M_1,\dots,M_{i-1}$ are determined, so
we can invoke \Cref{lem:proof_tilted_mgf} with $x=M_{i-1}$ and obtain
for $v:=\frac{c^2}{2}$
\begin{align}
  \log \E\bigl[ e^{t\|M_i\|^2} \bigm| \F_{i-1} \bigr]
  &=  \log \E\bigl[ e^{t\|D_i + M_{i-1}\|^2} \bigm| \F_{i-1} \bigr]
  \leq vt + \frac{t}{1-vt}\|M_{i-1}\|^2.
  \label{eq:proof_mgf_Mi}
\end{align}
In order to obtain a combined bound for any $\lambda\in(0,\frac{2}{nc^2})$,
set $\lambda_j := \frac{\lambda}{1-(n-j)v \lambda}$ for $j=0,\dots,n$.
First, observe that $v\lambda_j =
\frac{v\lambda}{1-(n-j)v\lambda}<\frac{1}{n-(n-j)}\leq 1$ for
$j=1,\dots,n$, because $v\lambda<\frac{2v}{nc^2} = \frac{1}{n}$.
Therefore, we can invoke~\eqref{eq:proof_mgf_Mi} with $t=\lambda_j$ and obtain
\begin{align}
  \log \E\bigl[ e^{\lambda_j \|M_j\|^2} \bigr] &= \log \E\Bigl[
  \E\bigl[ e^{\lambda_j \|M_j\|^2} \bigm| \F_{j-1} \bigr] \Bigr]
  \leq \lambda_j v + \log \E\Bigl[ e^{\lambda_{j-1} \|M_{j-1}\|^2} \Bigr],
\end{align}
where in the last step we have exploited that, by construction,
$\frac{\lambda_j}{1-\lambda_j v}=\lambda_{j-1}$ for any $j=1,\dots,n$.
Now, the final identity follows by iterating backwards from $M_n=F$
to $M_0=0$, yielding
\begin{align}
  \log \E\bigl[ e^{\lambda \|F\|^2} \bigr] &\leq \sum_{j=1}^n
  \lambda_j v = \sum_{k=0}^{n-1} \frac{\lambda v}{1-k\lambda v}
  \leq \int_{0}^n \frac{\lambda v}{1-\kappa\lambda v}d\kappa =
  -\log(1-nv\lambda),
\end{align}
where the second inequality holds because the function
$u\mapsto\frac{\lambda v}{1-u\lambda v}$ is monotonically increasing.
Substituting $v=\frac{1}{2}c^2$ completes the proof.
\end{proof}

We now have all components to complete the proof of \Cref{thm:maxinf-gauss-simple}:

\emph{Proof of~\Cref{thm:maxinf-gauss-simple} (continued). }
To derive a tail bound on $g(S,Z)$, we study its moment generating function.
\begin{align}
  \log\E e^{\lambda g(S,Z)} &= \log\E\bigl[ e^{\lambda(\langle
  \Psi,Z\rangle + \frac12\|\Psi\|^2)} \big]
  = \log\E\Bigl[ \E\bigl[ e^{\lambda \langle \Psi,Z\rangle +
\frac12\lambda\|\Psi\|^2} \bigr] \bigm| \Psi \bigr] \Bigr],
\intertext{where the last equality is due to the tower property of expectations.
  Because $Z$ is independent of all other variables, we can evaluate
the conditional expectation by means of \Cref{lem:gauss_mgf}}
&= \log\E e^{\frac{\lambda^2+\lambda}{2}\|\Psi\|^2}.
\intertext{By assumption, $\Psi$ fulfills a bounded difference
  inequality with parameter $\frac{1}{\sigma}$.
  Write $\nu=\frac{n}{4\sigma^2}$ and $r_\nu :=
  \sqrt{\frac{1}{4}+\frac{1}{\nu}}-\frac12$, such that
  $\lambda\in(0,r_{\nu})$ implies $\frac12(\lambda^2+\lambda) <
  \frac12(r_\nu+1)r_\nu = \frac{1}{2\nu}=\frac{2}{nc^2}$.
  Therefore, for $0\leq\lambda< r_{\nu}$, we can invoke Proposition
  \ref{prop:logmgf-squarednorm}
with $c=\frac{1}{\sigma}$ and obtain}
&\leq -\log(1-(\lambda^2+\lambda)\nu).  \label{eq:proof_mgf_as_log}
\end{align}

Next, we derive a tail bound of the type required for \Cref{lem:proof_gTZ}
following standard Chernoff arguments, \eg~\citep[Chapter 2]{boucheron2013}.
By Markov's inequality applied to $e^{\lambda g(S,Z)}$, one obtains
for all $\lambda>0$:
\begin{align}
&\Pr(g(S,Z)\geq \kappa)\leq e^{-\lambda \kappa}\E[e^{\lambda g(S,Z)}]
\intertext{To ensure that the right hand side remain below some
$\beta\in(0,1)$, it suffices that}
\kappa &= \inf_{\lambda>0} \ \frac{1}{\lambda}\Big( \log\E e^{\lambda
g(S,Z)} + L\Big) \label{eq:chernoff}
\intertext{with $L=\log(1/\beta)$.
  Inserting~\eqref{eq:proof_mgf_as_log} and restricting $\lambda$ to
the permitted range, we obtain}
&\leq \inf_{\lambda\in(0,r_{\nu})}\ \frac{1}{\lambda}\Big( -
\log(1-(\lambda+\lambda^2)\nu) + L \Big).
\label{eq:proof_mgf_g_as_log}
\end{align}
Furthermore, from the first statement of Proposition
\ref{prop:logmgf-squarednorm} we obtain $\frac12\E\|\Psi\|^2 \leq
\frac{1}{4}\frac{n}{\sigma^2}=\nu$. Inserting this and
~\eqref{eq:proof_mgf_g_as_log} into \Cref{lem:proof_gTZ} yields the
first statements of Theorem~\ref{thm:maxinf-gauss-simple}.

For the remaining statements, observe that
\begin{align}
1-(\lambda+\lambda^2)\nu &= (1 - a\lambda)(1 + b\lambda)
\label{eq:decompose_lambda_v}
\end{align}
for $a=\frac{\nu}{2}+\sqrt{\frac{\nu^2}{4}+\nu}$ and
$b=-\frac{\nu}{2}+\sqrt{\frac{\nu^2}{4}+\nu}$, which in particular
fulfills $a-b=\nu$ and $ab=\nu$. Therefore,
\begin{align}
-\log(1-(\lambda+\lambda^2)\nu) - \lambda\nu&=
-\log(1-a\lambda)-a\lambda  - \log(1+b\lambda) + b\lambda .
\label{eq:decompose_log_lambda_v}
\intertext{By the standard relations $-\log(1-u)-u\leq
  \frac{u^2}{2(1-u)}$ and $-\log(1+u)+u\leq \frac{u^2}{2}$,
and exploiting $a^2+b^2 = (a-b)^2+2ab = \nu^2+2\nu$, we obtain}
&\leq \frac{a^2\lambda^2}{2(1-a\lambda)} + \frac{b^2\lambda^2}{2}
\leq \frac{(a^2 + b^2)\lambda^2}{2(1-a\lambda)}
= \frac{(\nu^2+2\nu)\lambda^2}{2(1-a\lambda)}.
\end{align}
Therefore, we can upper bound the infimum:
\begin{align}
\inf_{\lambda\in(0,r_\nu)}\ &\frac{1}{\lambda}\Big(-\log(1-(\lambda+\lambda^2)\nu)
+ L \Big)
\leq \inf_{\lambda\in(0,r_\nu)}\Bigl[ \nu +
\frac{(\nu^2+2\nu)\lambda}{2(1-a\lambda)} + \frac{L}{\lambda}\Bigr]
\intertext{Taking derivative with respect to $\lambda$ on the right
  hand side, we find the exact
  minimizer $\lambda^* = \frac{\sqrt{2L}}{\sqrt{\nu^2+2\nu} +
  a\sqrt{2L}}$, which indeed fulfills
$0\leq\lambda^*< r_{\nu}$. Therefore,}
&\leq \nu + \frac{(\nu^2+2\nu)\lambda^*}{2(1-a\lambda^*)} + \frac{L}{\lambda^*}
\intertext{Observing $\frac{1}{\lambda^*}=a +
  \sqrt{\frac{\nu^2+2\nu}{2L}}$ and
$\frac{\lambda^*}{1-a\lambda^*}=\sqrt{\frac{2L}{\nu^2+2\nu}}$, it follows}
&=\nu  + \frac{(\nu^2+2\nu)}{2}\sqrt{\frac{2L}{\nu^2+2\nu}} +
L\Bigl(a + \sqrt{\frac{\nu^2+2\nu}{2L}}\Bigr)
\\
&=\nu  + \bigl(\onehalf\nu+\sqrt{\onequarter\nu^2+\nu}\bigr)L +
\sqrt{2(\nu^2+2\nu)L}.
\intertext{Exploiting $\sqrt{\onequarter\nu^2+\nu}\leq \onehalf\nu+\sqrt{\nu}$
  and $\sqrt{\nu^2+2\nu}\leq \nu+\sqrt{2\nu}$, we obtain a variant
  that is less tight
but can be sorted by powers of $\nu$:}
&\leq(1+L+\sqrt{2L})\nu + (L+2\sqrt{L})\sqrt{\nu}.
\end{align}
This proves the remaining statements of \Cref{thm:maxinf-gauss-simple}.
\end{proof} %

Finally, we also prove the claim that the optimization problem required
for evaluation~\eqref{eq:maxinf_gauss_as_inf} is convex.
\begin{lemma}\label{lem:proof_inf_convex}
For $\nu>0$, let $r_{\nu}:=\sqrt{\frac{1}{4}+\frac{1}{\nu}}-\frac12$. Then 
\begin{align}
f(\lambda) := \frac{1}{\lambda}\Bigl[ L - \log\bigl(1-(\lambda+\lambda^2)\nu\bigr) \Big]
\end{align}
is a smooth convex function on the interval $(0,r_\nu)$.
\end{lemma}

\begin{proof}
The smoothness of $f$ in the given interval is clear, 
because it is a composition of smooth functions and 
$\lambda\in(0,r_\nu)$ implies $(\lambda+\lambda^2)\nu<1$,
so $f$ is well-defined in the given interval. 
To show convexity, we decompose the log terms 
\begin{align}
    -\log\bigl(1-(\lambda+\lambda^2)\nu)
    &= -\log(1-a \lambda ) + -\log(1+b \lambda) 
\intertext{with explicit $a>b>0$ as for \eqref{eq:decompose_log_lambda_v} in the proof of \Cref{thm:maxinf-gauss-simple}. Now, expressing the logarithms by their power series, $-\log(1-x)=\sum_{k=1}^{\infty}x^k/k$, we obtain}
 &= \sum_{k=1}^{\infty} \frac{a^k + (-1)^k b^k}{k}\lambda^k,\label{eq:inf_powerseries}
\intertext{and therefore}
f(\lambda)    &=\frac{L}{\lambda} + \sum_{k=1}^{\infty} \frac{a^k + (-1)^k b^k}{k}\lambda^{k-1}.\label{eq:f_powerseries}
\end{align}
The second term is strictly convex, because $a>b>0$ implies that all coefficients $a^k + (-1)^k b^k$ of the power series are strictly positive and $\lambda\in(0,r_{\nu})$ in particular ensures $\lambda>0$. 
The first term, $\frac{L}{\lambda}$, is convex for all $\lambda>0$. Therefore, the claim follows.
\end{proof}

\subsection{Proof of Theorem~\ref{thm:maxinf_dpsgd}}

We now provide the proof of \Cref{thm:maxinf_dpsgd}, which we restate here for completeness.

\thmdpsgd*

\begin{proof}
Our proof follows similar steps to the proof of
\Cref{thm:maxinf-gauss-simple}, but now applied repeatedly, as
streaming DP-SGD is an iterative application of the Gaussian
mechanism. Care is still needed as the update vectors of different
batches are not independent of each other, as the model weights
change over time.
In addition, we are able to improve some constants compared to the
Gaussian case, because DP-SGD's feature function is a sum of
per-sample bounded vectors over a batch, which allows better
control of the concentration that just assuming a function of
bounded sensitivity as for \Cref{thm:maxinf-gauss-simple}.

For easier discussion, we prove this result for the 
general form DP-SGD that outputs each of its intermediate 
updates, not just the final model parameters. 
This is same approach used, for example, for proving 
its privacy properties. 
Specifically, this means that DP-SGD outputs its initialization, 
$\theta_0$, and each update vector $u_t$, from which the 
intermediate models, $\theta_t$, are computed.
The restricted form of DP-SGD that only outputs the final model 
is covered by this analysis, too, because it can be obtained by 
a postprocessing operation (ignore all outputs except the last one),
and can therefore not have higher approximate max-information.

As in the previous section, without loss of generality it 
suffices to prove the result for $\zeta=1/\sigma$, such that the
added noise has unit variance. The behavior for other choices
can be recovered by rescaling the loss function (which results
in gradients being scaled accordingly) and post-multiplication,
which does not influence the max-information.
However, we now cannot easily assume a fixed centered feature 
function by invoking global shift, because $\Psi$ depends on 
the previously computed model parameters and therefore changes 
between steps.

For a dataset $S=(B_1,\dots,B_T)$ of independently sampled data batches,
each of which consists of $m$ \iid sampled data points,
$B_t=(X^t_1,\dots,X^t_m)$,
let $Y=(Y_0,Y_1,Y_2,\dots,Y_T)$ with $Y_t=(\Theta_t,U_t)$, denote 
the output of the (extended) DP-SGD, where $Y_0=(\Theta_0,\emptyset)$ 
reflects the potentially independent random initialization. 
We write $Y_{s}^{t}=(Y_s,Y_{s+1},\dots,Y_t)$ for partial segments
of $Y$, and analogously $y_{s}^{t}=(y_s,\dots,y_t)$.

As in the previous section, it suffices to prove
$\Prob\{ \imath(S,Y)\geq \kappa\}\leq \beta$ for $\imath(S,Y) = \log
\frac{p(Y| S)}{p(Y)}$.
Analogously to Section~\ref{subsec:maxinf-gauss}, we do so by proving
a number of intermediate lemmas, which in the end we will combine to
establish the statement of Theorem~\ref{thm:maxinf_dpsgd}.

\begin{lemma}\label{lem:proof_fSY_to_hGZ}
  For any model parameter vector $\theta$, let
  $\Psi(x_1,\dots,x_m;\theta) = \sum_{i=1}^m\phi(x_i;\theta)$
  with $\phi(x;\theta)=\text{clip}\bigl(\nabla\ell(x;\theta), \zeta\bigr)$.
  For any $t\in\{1,\dots,T\}$, define random variables $\Psi_t :=
  \Psi(B_t;\Theta_{t-1})$
  and $G_t = \Psi_t - \E[\Psi_t | Y_0^{t-1}]$, and write $G=(G_1,\dots,G_T)$.
  Then, it holds for any $\kappa\geq 0$ that:
  \begin{align}
    \Pr\{\imath(S,Y) > \kappa\} &\leq \Pr\Bigl\{ g(G,Z) +
    \frac12\sum_{t=1}^T\E[\|G_t\|^2|Y_{0}^{t-1}] > \kappa \Bigr\}
    \label{eq:proofSYtoGZ}
    \intertext{for}
    g(G,Z) &:= \sum_{t=1}^{T} g_t(G_t,Z_t)\quad\text{with}\quad
    g_t(G_t,Z_t) := \bigl\langle Z_t, G_t\bigr\rangle + \onehalf \|G_t\|^2.
    \label{eq:proof_dpsgd_gt}
  \end{align}
    where $Z=(Z_1,\dots,Z_T)$ consists of standard multivariate Gaussians, 
    such that each $Z_t$ is \emph{fresh}, \ie independent of the data, 
    the initialization and the previous noise variables $Z_1,\dots,Z_{t-1}$. 
\end{lemma}

\begin{proof}
  Let $\tilde S=(\tilde B_1,\dots,\tilde B_T)$ be an independent copy of $S$,
  and denote by $\tilde Y$ the output of the algorithm on $\tilde S$.
  By Jensen's inequality applied to $t\mapsto \log\frac{1}{t}$, we
  have for any output $y=(y_1,\dots,y_T)$:
  \begin{align}
    \imath(S,Y=y) &\leq \E_{\tilde S}\log \frac{p(Y=y| S)}{p(\tilde
    Y=y| \tilde S)}
    =
    \E_{\tilde S}\log \prod_{t=1}^T
    \frac{p(Y_t=y_t|Y_{0}^{t-1}=y_{0}^{t-1},S)}{p(\tilde
    Y_t=y_t|\tilde Y_{0}^{t-1}=y_{0}^{t-1},\tilde S)}
    \\
    &= \sum_{t=1}^T \E_{\tilde S}\log
    \frac{p(Y_t=y_t|Y_{0}^{t-1}=y_{0}^{t-1},S)}{p(\tilde
    Y_t=y_t|\tilde Y_{0}^{t-1}=y_{0}^{t-1},\tilde S)}.
    \intertext{For each $t\in\{1,\dots,T\}$, DP-SGD computes the
      model update $\Theta_t$
      deterministically from the current update vector $U_t$ and the past
    model parameters $\Theta_{0}^{t-1}$. Therefore, with $y_t=(u_t,\theta_t)$, }
    &=\sum_{t=1}^T \E_{\tilde S}\log
    \frac{p(U_t=u_t|Y_{0}^{t-1}=y_0^{t-1},S)}{p(\tilde U_t=u_t|\tilde
    Y_{0}^{t-1}=y_0^{t-1},\tilde S)}.%
    \intertext{The model update, $U_t$, is the result of an application
      of the Gaussian mechanism, so its distribution, conditioned on any
      value for $Y_0^{t-1}$ and $S$, is Gaussian, namely
      $\mathcal{N}(\Psi(B_t;\Theta_{t-1}),\text{I})$.
    The analog property holds for $\tilde U_t$. Consequently,}
    &=\sum_{t=1}^T \E_{\tilde S}\log
    \frac{\mathcal{N}(u_t;\Psi(B_t;\theta_{t-1}),\text{I})}{\mathcal{N}(u_t;\Psi(\tilde
    B_t;\theta_{t-1}),\text{I})},
    \label{eq:proof_ratiodecomp}
  \end{align}
  and we can analyze each term in \eqref{eq:proof_ratiodecomp} in the same way
  as we had in the case of a single application of the Gaussian mechanism in
  the previous section. In particular, because of the independence
  between batches,
  we can replace the expectation with respect to $\tilde S$ by the expectation
  with respect to the respective $\tilde B_t$.

  Analogously to the proof of \Cref{lem:proof_gTZ}, we obtain for any
  batches $B$
  and $\tilde B$, any vector $u\in\R^d$ and conditioned on any model
  parameters $\Theta\in\R^d$ that:
  \begin{align}
    \E_{\tilde B} &\log \frac{\mathcal{N}(u;\,\Psi(B;\Theta),\text{I})}
    {\mathcal{N}(u;\,\Psi(\tilde B;\Theta),\text{I})}
    = \bigl\langle u-\Psi(B;\Theta),
    \Psi(B;\Theta)-\E_B[\Psi(B;\Theta)]\bigr\rangle  \notag
    \\
    &\qquad+ \frac12\|\Psi(B;\Theta)-\E_B[\Psi(B;\Theta)]\|^2 +
    \frac12\E_B\|\Psi(B;\Theta)-\E_B[\Psi(B;\Theta)]\|^2,
  \end{align}
  where we have used that $\E_B[\Psi(B;\Theta)|\Theta]=\E_{\tilde
  B}[\Psi(\tilde B;\Theta)|\Theta]$.

  We adopt the notation $\Psi_t=\Psi(B_t;\Theta_{t-1})$ and $G_t =
  \Psi_t - \E[\Psi_t | Y_0^{t-1}]$
  from the Lemma statement and insert $U_t$ for $u$.
  Then, in continuation of \eqref{eq:proof_ratiodecomp}, we have:
  \begin{align}
    \imath(S,Y) &\leq \sum_{t=1}^T \bigl\langle U_t-\Psi_t,
    G_t\bigr\rangle + \frac{1}{2}\|G_t\|^2 + \frac{1}{2}\E_S\|G_t\|^2.
  \end{align}
  Observing that each $U_t-\Psi_t(S)$ has the distribution of a
  standard Gaussian, $Z_t$,
  that is independent of all other variables, we obtain the statement of
  Lemma~\ref{lem:proof_fSY_to_hGZ}.
\end{proof}

We also prove refined versions of \Cref{lem:proof_tilted_mgf} and
Proposition \ref{prop:logmgf-squarednorm}
for sums of bounded vectors rather than just functions that fulfill a
bounded difference inequality.
\begin{lemma}\label{lem:proof_tilted_mgf_refined}
  Let $V$ be a vector-valued random variable that fulfills $\|V\|\leq
  r$ almost surely,
  and let $\bar V=V-\E[V]$. Then it holds for any vector $x$ and any
  $0\leq t<\frac{1}{2r^2}$
  \begin{equation}
    \log\E e^{t\|x+\bar V\|^2} \leq r^2t + \frac{t}{1-2r^2t}\|x\|^2.
  \end{equation}
\end{lemma}
Note that the support of $\bar V$ has at most diameter $2r$, so
\Cref{lem:proof_tilted_mgf}
would guarantee $\log\E e^{t\|x+\bar V\|^2} \leq 2r^2t +
\frac{t}{1-2r^2t}\|x\|^2$. In contrast,
\Cref{lem:proof_tilted_mgf_refined} improves the leading term by
exploiting that $\bar V$ stems from a variable with norm bounded by $r$.

\begin{proof}
  The proof mostly follows the same steps as for \Cref{lem:proof_tilted_mgf},
  so we only detail where differences occur.
  Specifically, since we know $\|V\|\leq r$, we obtain
  \begin{equation}
    \E_Q\|\bar V-\E_Q[\bar V]\|^2 = \E_Q\|V-\E_Q[V]\|^2 = \E_Q\|V\|^2
    - \|\E_Q V\|^2 \leq r^2.
  \end{equation}
  This is in contrast to \eqref{eq:proof_var_V_under_Q}, where we had
  to argue with the diameter
  instead of the radius, resulting in an upper bound of
  $\frac12c^2=2r^2$, \ie twice as big.

  Consequently, we obtain as analog of
  \eqref{eq:proof_squared_norm_mgf_shift}, for all $x$:
  \begin{equation}
    \E_Q\|x+\bar V\|^2 \leq \|x+\E_Q[\bar V]\|^2 + r^2.
    \label{eq:proof_squared_norm_mgf_shift_refined}
  \end{equation}

  The other steps cannot be improved as easily, in particular
  the bound $\KL(Q\|P)\geq \frac{\|\E_Q[\bar V]\|^2}{2r^2}$ is the
  same as \eqref{eq:proof_squared_norm_mgf_KL} for $c=2r$.
  Combining the terms and optimizing, it holds for any $x$:
  \begin{align}
    \log \E_P e^{t\|x+\bar V\|^2} &\leq r^2t + t\|x+\E_Q[\bar V]\|^2
    - \frac{\|E_Q[\bar V]\|^2}{2r^2}.
    \intertext{Expanding the norms and completing the square, it follows}
    &\leq r^2t + \frac{t}{1-2r^2t}\|x\|^2,
  \end{align}
  because $t<\frac{1}{2r^2}$ by assumption. This completes the proof
  of \Cref{lem:proof_tilted_mgf_refined}.
\end{proof}

As a refined version of Proposition \ref{prop:logmgf-squarednorm} we obtain
\begin{proposition}\label{prop:logmgf-squarednorm-refined}
  Let $X_1,\dots,X_m$ be independent random variables that fulfill
  $\|X_i\|\leq r$ almost surely,
  and write $X=(X_1,\dots,X_m)$. Let $f(X)=\sum_{i=1}^m X_i$. Then,
  for all $0\leq\lambda<\frac{1}{2mr^2}$,
  it holds that
  \begin{align}
    \E[\|f(X)-\E[f(X)]\|^2] &\leq mr^2
    \label{eq:squared_norm_variance_refined}
    \\
    \log\E e^{\lambda\|f(X)-\E[f(X)]\|^2} &\leq
    -\frac12\log(1-2mr^2\lambda). \label{eq:squared_norm_logmgf_refined}
  \end{align}
\end{proposition}

Comparing \eqref{eq:squared_norm_variance_refined} to the previous
\eqref{eq:squared_norm_variance}
with $c=2r$, one sees an improvement by a factor of two in the leading constant.
Similarly, the leading constant of
\eqref{eq:squared_norm_logmgf_refined} is half of the leading
constant in \eqref{eq:squared_norm_logmgf}, but the terms inside the
logarithms remain as they would be for $c=2r$.

\begin{proof}
  The proof of \eqref{eq:squared_norm_variance_refined} is standard.
  By the additive form of $f$ and the independence condition, it holds that
  \begin{align}
    \E[\|f(X)-\E[f(X)]\|^2] &= \sum_{i=1}^m \E\|X_i-\E[X_i]\|^2
    \leq \sum_{i=1}^m \E\|X_i\|^2 \leq mr^2.
  \end{align}

  The proof of \eqref{eq:squared_norm_logmgf_refined} follows the same steps
  as for Proposition \ref{prop:logmgf-squarednorm}, except simplified due to the
  additive structure, and using \Cref{lem:proof_tilted_mgf_refined}
  in place of \Cref{lem:proof_tilted_mgf}.

  Writing $\bar X_j = X_j-\E[X_j]$ and $F_j:=\sum_{i=1}^j\bar X_i$
  for $j=1,\dots,m$, we obtain for $v=2r^2$:
  \begin{align}
    \log \E\bigl[ e^{t\|F_j\|^2} \bigm| F_{j-1} \bigr]
    &=  \log \E\bigl[ e^{t\|\bar X_j + F_{j-1}\|^2} \bigm| \F_{j-1} \bigr]
    \leq \frac12 vt + \frac{t}{1-vt}\|F_{j-1}\|^2.
  \end{align}
  because $F_{j-1}$ is determined by the conditioning and $\bar X_j$
  fulfills the conditions of \Cref{lem:proof_tilted_mgf_refined}.

  Employing the same iterative construction as in the proof of
  Proposition \ref{prop:logmgf-squarednorm},
  with $\lambda_j = \frac{\lambda}{1-(m-j)v\lambda}$ but now with
  $v=2r^2$ as above
  and for $\lambda\in(0,\frac{1}{2mr^2})$, we obtain
  \begin{align}
    \log \E\bigl[ e^{\lambda_j\|F_j\|^2} \bigr]
    &=  \log \E\Bigl[ \E\bigl[ e^{\lambda_j\|F_j\|^2} \bigm| F_{j-1}
    \bigr] \Bigr]
    \leq \frac12 v\lambda_j + \log\E\bigl[e^{\lambda_{j-1}\|F_{j-1}\|^2}\bigr],
    \intertext{and ultimately, by iterating backward:}
    \log \E\bigl[ e^{\lambda\|F_m\|^2} \bigr] &\leq \sum_{j=1}^m \frac12\lambda_j v
    \leq -\frac12\log(1-mv\lambda),
  \end{align}
  which concludes the proof.
\end{proof}

\emph{Proof of \Cref{thm:maxinf_dpsgd} (continued). } 
We now establish a bound on the moment-generating function of
$\sum_{t=1}^T g_t(G_t,Z_t)$, as
introduced in \eqref{eq:proof_dpsgd_gt}, exploiting the tower
property of expectations.

For any $1\leq t\leq T$, let $\mathcal{F}_{t-1} =
\sigma(\Theta_0,B_1,\dots,B_{t-1},Z_1,\dots,Z_{t-1})$. Then any
partial sum, $\sum_{j=1}^{t-1} g_j(Z_j, G_j)$, is measurable with
respect to this filtration, and $(B_t,Z_t)$ are independent of it. Consequently,
it holds for the partial sum:
\begin{align}
  \E\Bigl[e^{\lambda\sum\limits_{j=1}^t g_j(Z_j, G_j)} \Bigr]
  &= \E\Bigl[e^{\lambda\sum\limits_{j=1}^{t-1} g_j(Z_j, G_j) +
      \lambda\langle Z_t, G_t\rangle
  + \frac{\lambda}{2}\|G_t\|^2} \Bigr]
  \\
  &= \E\Bigl[e^{\lambda\sum\limits_{j=1}^{t-1} g_j(Z_j, G_j)}
    \E\Bigl[\E\bigl[e^{\lambda\langle Z_t, G_t \rangle}|B_t,
      \mathcal{F}_{t-1}] \ e^{\frac{\lambda}{2}\|G_t\|^2} \Bigm|
  \mathcal{F}_{t-1}\bigr]\Bigr].
  \intertext{By \Cref{lem:gauss_mgf}, it follows that}
  & = \E\Bigl[e^{\lambda\sum\limits_{j=1}^{t-1} g_j(Z_j, G_j)}
    \E\bigl[e^{\frac{\lambda + \lambda^2}{2}\|G_t\|^2} \bigm|
  \mathcal{F}_{t-1}\bigr]\Bigr]
  \intertext{Remember that $G_t =
    \Psi(B_t,\Theta_{t-1})-\E[\Psi(B_t,\Theta_{t-1})|Y_{0}^{t-1}]$,
    so conditioned on $\mathcal{F}_{t-1}$, its randomness is only due
    to $B_t$, which
    contains independent elements. Consequently, $G_t$ fulfills the
    conditions of Proposition \ref{prop:logmgf-squarednorm-refined}
    with $r=1/\sigma$, and we obtain, as long as
  $\frac{\lambda+\lambda^2}{2}<\frac{\sigma^2}{2m}$:}
  & \leq \E\Bigl[e^{\lambda\sum\limits_{j=1}^{t-1} g_j(Z_j, G_j)}
  \frac{1}{(1-\frac{m}{\sigma^2}(\lambda+\lambda^2))^{1/2}}\Bigr].
  \intertext{Applying the above identity iteratively, we obtain a
  bound for the complete expression:}
  \E\Bigl[e^{\lambda\sum\limits_{j=1}^T g_j(Z_j, G_j)} \Bigr] &\leq
  \frac{1}{(1-\frac{m}{\sigma^2}(\lambda+\lambda^2))^{T/2}}.
  \label{eq:proof_dpsgd_mgfbound}
\end{align}

To obtain tail bounds, we follow the same Chernoff arguments as in
the previous section, only with slightly different constants:
\begin{align}
  &\Pr(g(G,Z)\geq \kappa)\leq e^{-\lambda \kappa}\E[e^{\lambda g(G,Z)}]
  \intertext{With $L=\log(1/\beta)$ for any $\beta\in(0,1)$:}
  \kappa &= \inf_{\lambda> 0} \ \frac{1}{\lambda}\Big( \log\E
  e^{\lambda g(G,Z)} + L\Big) \label{eq:proof_dpsgd_chernoff}
  \intertext{Analogous to~\eqref{eq:proof_mgf_g_as_log} we obtain,
  with $r_{4\nu}=\frac12\sqrt{1 + \frac{1}{\nu}}-\frac12$}
  &\leq \inf_{\lambda\in(0,r_{4\nu})}\ \frac{1}{\lambda}\Big( -
  \frac{T}{2}\log(1-4(\lambda+\lambda^2)\nu) + L \Big).
  \label{eq:proof_dpsgd_kappa_by_log}
\end{align}
Furthermore, from Proposition \ref{prop:logmgf-squarednorm-refined}
it follows that $\frac12\E[\|G_t\|^2|Y_{0}^{t-1}] \leq
\frac{m}{2\sigma^2} = 2\nu$ for all $t=1,\dots,T$, so
$\frac12\sum_{t=1}^T\E[\|G_t\|^2|Y_{0}^{t-1}] \leq 2T\nu$.
Inserting this and~\eqref{eq:proof_dpsgd_kappa_by_log} into
\Cref{lem:proof_fSY_to_hGZ} yields the first statements of
Theorem~\ref{thm:maxinf_dpsgd}.

For the remaining statements, note that the logarithmic term,
$\log\bigl(1-4(\lambda+\lambda^2)\nu\bigr)$, has the same form as
\eqref{eq:proof_mgf_g_as_log} in case of the Gaussian mechanism, only
with $\nu$ replaced by $4\nu$. The additional factor $T/2$ in front of 
it can be moved outside of the infimum by changing $L$ to $2L/T$.
Subsequently, we can reuse the steps from the previous section and obtain
\begin{align}
  \inf_{\lambda\in(0,r_{4\nu})}\ &\frac{1}{\lambda}\Big(
  -\frac{T}{2}\log(1-4(\lambda+\lambda^2)\nu) + L \Big)
  \\
  &\leq 2T\nu  + 2\bigl(\nu+\sqrt{\nu^2+\nu}\bigr)L + 2\sqrt{(4\nu^2+2\nu)TL},
  \intertext{and, simplified further and sorted by powers of $\nu$,}
  &\leq (2T+4L+4\sqrt{TL})\nu + 2(L+\sqrt{2TL})\sqrt{\nu}.
\end{align}
In combination with $\frac12\sum_{t=1}^T\E[\|G_t\|^2|Y_{0}^{t-1}]
\leq 2T\nu$ as above, the remaining statements of
\Cref{thm:maxinf_dpsgd} follow.
\end{proof}

Finally, we also formally state the fact that the optimization problem in \eqref{eq:maxinf_dpsgd} is strictly convex.
\begin{lemma}
For $\nu>0$, let $r_{4\nu}:=\frac12\sqrt{1+\frac{1}{\nu}}-\frac12$. Then 
\begin{align}
f(\lambda) := \frac{1}{\lambda}\Bigl[ L  - \frac{T}{2}\log\bigl(1-4(\lambda+\lambda^2)\nu\bigr) \Big]
\end{align}
is a smooth strictly convex function on the interval $(0,r_{4\nu})$.
\end{lemma}
\begin{proof}
The proof is analogous to \Cref{lem:proof_inf_convex}, since neither the switch from $\nu$ to $4\nu$ nor the multiplication by $T/2$ does change the convexity properties of the function.
\end{proof}

\subsection{Proof of Theorem~\ref{thm:maxinf_from_renyi}}\label{sec-app:simple-rdp}

We first prove the linear bound in \Cref{thm:rdp-linear}, then strengthen it in the strong privacy regime in \Cref{rdp:thm:app}.
Together, these two theorems prove \Cref{thm:maxinf_from_renyi}.

\begin{theorem}
  [Max-Information bound for RDP]
  \label{thm:rdp-linear}
  Suppose the dataset $S=(X_1,\dots,X_n) \sim \pointDistribution^n$
  is drawn with \iid entries, and $\A$ is an $(\rdpOrder,\rdpValue)$-RDP
  mechanism for some $\rdpOrder>1$ and $\rdpValue\geq0$. Then, for every
  $\beta \in (0,1)$, we have
  \begin{equation}
    \label{eq:renyi-dp-max-info-bound}
    I^{\beta}_{\infty}(\A(S),S)
    \leq n\left(\rdpValue+\frac{\log(1/\beta)}{\rdpOrder-1}\right).
  \end{equation}
\end{theorem}

\begin{proof}
We begin by recalling the Hellinger integral, which is jointly convex in its
two arguments \citep{ervenRenyiDivergenceKullbackLeibler2014}:
\begin{equation}
  \Hell_\rdpOrder(P\|Q) = \exp\bigl((\rdpOrder-1)
  \renyiDiv_\rdpOrder(P\|Q)\bigr) = \int
  (p)^\rdpOrder  (q)^{1-\rdpOrder}.
\end{equation}

Define a sequence of random variables $Y_i : \mathcal{X}^{i} \to
\mathcal{Y}$ which take the first $i$ points as input, and draw the
remaining $n-i$ from $\pointDistribution$ i.i.d.:
\begin{equation}
  Y_i(x_{1:i})
  = \A(S)\mid X_1=x_1,\dots,X_i=x_i,
  \qquad S=(X_1,\dots,X_n)\sim\pointDistribution^n.
\end{equation}

Pairs of these random variables satisfy two properties concerning
their Rényi divergences. Firstly,
\begin{equation}
  \renyiDiv_\rdpOrder(Y_i(x_{1:i-1},x_i)
  \|Y_i(x_{1:i-1},x'_i)) \leq \rdpValue,
\end{equation}
which holds by the joint convexity of $\Hell_\rdpOrder$, and the fact that
$(x_{1:i-1},x_i,x_{i+1:n})$ and
$(x_{1:i-1},x'_i,x_{i+1:n})$ are neighboring datasets for every
suffix $x_{i+1:n}$.

Secondly, incremental steps satisfy the same property:
\begin{equation}
  \renyiDiv_\rdpOrder(Y_i(x_{1:i-1},x_i)
  \|Y_{i-1}(x_{1:i-1})) \leq \rdpValue,
\end{equation}
because $Y_{i-1}(x_{1:i-1})$ is the mixture of
$Y_i(x_{1:i-1},x'_i)$ over $x'_i\sim\pointDistribution$, and
$\Hell_\rdpOrder$ is convex in its second argument.

Consider the information-density increment obtained by
conditioning on one extra data point:
\begin{equation}
  \Lambda_i = \log
  \frac{p_{Y_i}(Y\mid X_{1:i})}
  {p_{Y_{i-1}}(Y\mid X_{1:i-1})};
  \qquad S=(X_1,\dots,X_n)\sim\pointDistribution^n,
  \quad Y=\A(S).
\end{equation}

The $(\rdpOrder,\rdpValue)$ guarantee controls the Hellinger integral at
the order $\rdpOrder$:
\begin{align}
  \E \exp\bigl((\rdpOrder-1) \Lambda_i \bigr)
  & =
  \E \left[\E\left[ \exp\bigl((\rdpOrder-1) \Lambda_i\bigr) \mid
  X_{1:i}\right]\right]
  \nonumber\\
  & =
  \E \left[\Hell_{\rdpOrder}(Y_{i}(X_{1:i}) \| Y_{i-1}(X_{1:i-1}))\right]
  \nonumber\\
  & \leq
  \E \left[\exp\bigl((\rdpOrder-1) \rdpValue\bigr) \right]
  \nonumber\\
  &= \exp\bigl((\rdpOrder-1)\rdpValue\bigr).
  \label{eq:one-step-moment-ish-bound}
\end{align}

The information density for $\A$ is
\begin{equation}
  \info=\log\frac{p_{Y\mid S}(Y\mid S)}{p_Y(Y)} = \sum_{i=1}^n\Lambda_i
\end{equation}
Set $\lambda:=\frac{\rdpOrder-1}{n}$, so that $n\lambda=\rdpOrder-1$, and
bound the moments of $\info$ using the generalized H\"older's inequality:

\begin{align}
  \E \exp \left(
    \lambda \info
  \right)
  &=
  \E
  \prod_{i=1}^n
  \exp \left( \lambda \Lambda_i  \right)
  \leq
  \prod_{i=1}^n
  \left( \E
    \exp \left(n \lambda \Lambda_i
    \right)
  \right)^{1/n}
  \\
  &=
  \prod_{i=1}^n
  \left( \E
    \exp \bigl((\rdpOrder-1) \Lambda_i
    \bigr)
  \right)^{1/n}
  \leq
  \exp\bigl((\rdpOrder-1)\rdpValue\bigr),
\end{align}
where the last inequality is due to \eqref{eq:one-step-moment-ish-bound}.
We set
\begin{equation}
  \kappa := n\rdpValue+\frac{n}{\rdpOrder-1}\log\frac1\beta,
\end{equation}
and apply Chernoff's inequality with the same $\lambda=\frac{\rdpOrder-1}{n}$:
\begin{equation}
  \label{eq:rdp-tail-bound}
  \Prob(\info > \kappa )
  \leq \exp(-\lambda\kappa)\,\E\exp(\lambda\info)
  \leq \exp\bigl(-\lambda\kappa+(\rdpOrder-1)\rdpValue\bigr)
  = \beta.
\end{equation}

An application of Lemma 18 from \cite{dwork2015generalization}
completes the proof of
\cref{eq:renyi-dp-max-info-bound}. 

\end{proof}

\newcommand{\cF}{\mathcal F}

\begin{theorem}\label{rdp:thm:app}
  Assume that $\A$ is  $(\rdpOrder,\rdpValue)$-RDP for some
  $\rdpOrder>1$ and $\rdpValue \ge 0$, and $S$ is a dataset of $n$
  \iid samples. Fix \(\beta_1,\beta_2>0\) with
  \(\beta_1+\beta_2=\beta<1\), write \(L_2:=\log(1/\beta_2)\), and define
  \begin{equation}\label{eq:psi-v}
    \psi(u):=\log^2\!\bigl(\sqrt u+\sqrt{1+u}\bigr)=\arcsinh^2\!\sqrt u,\qquad
    v_\star:=\frac{n}{2(\rdpOrder-1)^2}\,
    \psi\!\left(\frac{e^{(\rdpOrder-1)\rdpValue}-1}{\beta_1}\right).
  \end{equation}
  Then
  \begin{equation}\label{eq:optimized}
    I_\infty^\beta(\A(S),S)
    \le v_\star+2\sqrt{v_\star L_2}
    +\frac{L_2}{\rdpOrder-1}.
  \end{equation}

  Moreover,
  \begin{equation}\label{eq:v-def}
    v_\star\le v:=\frac n8\left(
      \rdpValue+\frac{\log(4/\beta_1)}{\rdpOrder-1}
    \right)^2,
  \end{equation}
  and the right-hand side of \eqref{eq:optimized} is increasing in
  \(v_\star\), so \eqref{eq:optimized} remains valid with \(v\) in
  place of \(v_\star\).
\end{theorem}

The proof is organized as follows.  \Cref{rdp:lem:spread} bounds, for a
single coordinate, the symmetrized R\'enyi divergence between the
posterior of \(X_i\) and the prior; this is the only place the privacy
assumption is used.
\Cref{rdp:lem:corrected-moment} then composes the \(n\) coordinates into
one moment bound for the information density.  The proof of the theorem
combines both lemmas, excluding a bad event.

\paragraph{Setup and notation}

Throughout, \(Y=\A(S)\), and we write
\begin{equation}\label{eq:PQ-def}
  P:=P_{S,Y},\qquad
  Q:=P_S\otimes P_Y=\pointDistribution^{n}\otimes P_Y
\end{equation}
for the joint distribution of the data and the output and for the product of its
marginals, with \(\psi\) and \(v_\star\) as in \eqref{eq:psi-v}.

As in the \((\rdpValue,\delta)\)-DP proof, let
\begin{equation}\label{eq:filtration}
  \cF_i:=\sigma(Y,X_1,\ldots,X_i),\qquad
  \cF_0:=\sigma(Y),
  \qquad
  P_i:=P_{X_{\le i},Y},\qquad P_0:=P_Y.
\end{equation}
For a fixed history \((y,x_{<i})\), define the posterior-to-prior
ratio for datapoint $x_i$:
\begin{equation}\label{eq:ratio}
  r_i^{y,x_{<i}}(x)
  :=\frac{dP_i}{d(P_{i-1}\otimes\pointDistribution)}(x_{<i},x,y)
  =\frac{dP_{X_i\mid Y=y,X_{<i}=x_{<i}}}{d\pointDistribution}(x).
\end{equation}
When the history is random, write
\begin{equation}
  r_i(x):=r_i^{Y,X_{<i}}(x).
\end{equation}
Thus \(r_i(\cdot)\) depends only on \(\cF_{i-1}\), and
\begin{equation}\label{eq:ratio-normalized}
  r_i(x)>0 \text{ a.e.},\qquad \int r_i(x)\pointDistribution(dx)=1.
\end{equation}

As before, write
\(\Hell_\rdpOrder(U\|V):=\exp\bigl\{(\rdpOrder-1)\renyiDiv_\rdpOrder(U\|V)\bigr\}\)
for the
Hellinger integral of order \(\rdpOrder\), and, for two probability measures
\(U,V\), write
\begin{equation}\label{eq:sym-renyi}
  \bar \renyiDiv_\rdpOrder(U,V)
  :=\renyiDiv_\rdpOrder(U\|V)+\renyiDiv_\rdpOrder(V\|U)
\end{equation}
for the \emph{symmetrized R\'enyi divergence} of order
\(\rdpOrder\), the analogue of the Jeffreys divergence. A similar construction is given by \cite{talih2013reference}.  For a
positive density \(w\) with respect to \(\pointDistribution\), let \(W\)
be the probability measure with \(dW=w\,d\pointDistribution\) and put
\begin{equation}\label{eq:J-def}
  J(w):=\exp\bigl\{(\rdpOrder-1)\bar
  \renyiDiv_\rdpOrder(W,\pointDistribution)\bigr\}
  =\Hell_\rdpOrder(W\|\pointDistribution)\,\Hell_\rdpOrder(\pointDistribution\|W)
  =\left(\int w^{\rdpOrder}\,d\pointDistribution\right)
  \left(\int w^{1-\rdpOrder}\,d\pointDistribution\right),
\end{equation}
and \(J_i:=J(r_i)\).  The last equality is immediate from
\(d\pointDistribution/dW=1/w\).  Since the R\'enyi divergence is
nonnegative, \(J_i\ge1\) and each \(J_i\) is \(\cF_{i-1}\)-measurable.

The chain rule for the posterior distribution gives
\begin{equation}\label{eq:factorization}
  \frac{dP}{dQ}(s,y)
  =\frac{dP_{S\mid Y=y}}{d\pointDistribution^{n}}(s)
  =\prod_{i=1}^n r_i^{y,x_{<i}}(x_i).
\end{equation}
Consequently, for the information density we have:
\begin{equation}\label{eq:information-density}
  \info
  :=\log\frac{dP}{dQ}(S,Y)
  =\sum_{i=1}^n\log r_i(X_i).
\end{equation}

\paragraph{Symmetrized R\'enyi divergence of a single coordinate}

\begin{lemma}\label{rdp:lem:spread}
  For every \(i\in\{1,\ldots,n\}\),
  \begin{equation}\label{eq:expected-spread}
    \E_P J_i\le e^{(\rdpOrder-1)\rdpValue}.
  \end{equation}
\end{lemma}

\begin{proof}
  Fix a coordinate \(i\) and a prefix \(u=x_{<i}\), and write
  \begin{equation}\label{eq:K-def}
    P^u_x:=P_{Y\mid X_{<i}=u,\,X_i=x}
    =\int \A\bigl((u,x,z),\cdot\bigr)\,\pointDistribution^{n-i}(dz),
    \qquad
    P^u:=\int P^u_{x}\,\pointDistribution(dx),
  \end{equation}
  for the conditional distribution of the output given the prefix and the
  current observation \(P^u_x\), and given the prefix alone, \(P^u\).

  By the same logic as in the simpler result, the datasets
  \((u,x,z)\) and \((u,x',z)\)
  are neighbors for any fixed suffix \(z\), so
  \begin{equation}\label{eq:mixture-rdp}
    \Hell_{\rdpOrder}(P^u_x\|P^u_{x'})\le e^{(\rdpOrder-1)\rdpValue}
    \qquad\text{for all }x,x'.
  \end{equation}
  holds using the joint convexity of the Hellinger integral and the
  privacy guarantee on the mechanism.

  The posterior-to-prior ratio \(r_i^{y,u}(x)\) \eqref{eq:ratio} is a
  density in \(\mathcal X\) for
  fixed \(y\), whereas \eqref{eq:mixture-rdp} is a bound on densities in
  \(\mathcal Y\) for fixed \(x\).  We can convert between these two
  functions using Bayes' rule, showing that (almost everywhere),
  \begin{equation}\label{eq:r-equals-k}
    r_i^{y,u}(x)=\frac{dP^u_x}{dP^u}(y).
  \end{equation}

  Now work conditionally on \(X_{<i}=u\), so that the output is distributed as
  \(P^u\).  Substituting \eqref{eq:r-equals-k} into the definition
  \eqref{eq:J-def} of \(J_i\) and exchanging the nonnegative integrals by
  Tonelli's theorem,
  \begin{align}
    \E_P[J_i\mid X_{<i}=u]
    &=\iint\left[\int
      \left(\frac{dP^u_x}{dP^u}\right)^{\rdpOrder}
    \left(\frac{dP^u_{x'}}{dP^u}\right)^{1-\rdpOrder}\!\!dP^u\right]
    \pointDistribution(dx)\pointDistribution(dx') \notag\\
    &=\iint \Hell_{\rdpOrder}(P^u_x\|P^u_{x'})\,
    \pointDistribution(dx)\pointDistribution(dx')
    \;\le\;e^{(\rdpOrder-1)\rdpValue},
    \label{eq:spread-prefix}
  \end{align}
  the last step by \eqref{eq:mixture-rdp}.  Averaging over the prefix proves
  \eqref{eq:expected-spread}.
\end{proof}

\paragraph{Moment bound on information density }
\Cref{rdp:lem:spread} controls each coordinate in expectation.  The
argument requires a moment bound on the information density
\(\info\). We show this in terms of variables $v_i$,
defined using \(\psi\) as in \eqref{eq:psi-v} :
\begin{equation}\label{eq:v-i-def}
  v_i:=\frac{\psi(J_i-1)}{2(\rdpOrder-1)^2},
\end{equation}

\begin{lemma}\label{rdp:lem:corrected-moment}
  For every \(0<\lambda<\rdpOrder-1\), writing
  \(b_\lambda:=\lambda+\frac{\lambda^2}{1-\lambda/(\rdpOrder-1)}\),
  \begin{equation}\label{eq:corrected-moment}
    \E_P\exp\left\{
      \lambda\info-b_\lambda\sum_{i=1}^n v_i
    \right\}\le1.
  \end{equation}
\end{lemma}

\begin{proof}
  Fix \(i\) and a history \((y,x_{<i})\), and suppress the history in the
  notation \(r_i\).  By \eqref{eq:ratio-normalized}, \(r_i\) is a
  probability density with respect to \(\pointDistribution\), positive
  almost everywhere, and \(J(r_i)=J_i\) by definition.  \Cref{lem:auxiliary}
  applied with \(w=r_i\) and \(M=J_i\), so that \(v_w=v_i\), gives
  \begin{equation}\label{eq:one-coordinate-moment}
    \log\int r_i(x)^{1+\lambda}\pointDistribution(dx)
    \le v_i\left(
      \lambda+\frac{\lambda^2}{1-\lambda/(\rdpOrder-1)}
    \right).
  \end{equation}
  Hence, at every fixed history,
  \begin{equation}\label{eq:coordinate-factor}
    \int r_i(x)^{1+\lambda}e^{-b_\lambda v_i}\pointDistribution(dx)\le1.
  \end{equation}

  We now integrate the coordinates backwards.  Changing measure with
  \(dP/dQ=e^{\info}\) and using the factorisation and decomposition
  \eqref{eq:factorization}, the left-hand side of
  \eqref{eq:corrected-moment} equals
  \begin{equation}\label{eq:product-integral}
    \int\prod_{i=1}^n
    \left[r_i(X_i)^{1+\lambda}e^{-b_\lambda v_i}\right]dQ.
  \end{equation}
  Under \(Q\), hold \(Y=y\) fixed and integrate
  first over \(X_n\).  Every earlier factor is independent of \(x_n\), and
  \(v_n\) depends only on \((y,x_{<n})\).  Hence the integral over \(x_n\)
  is
  \begin{equation}
    \left(\prod_{i=1}^{n-1}
    r_i(x_i)^{1+\lambda}e^{-b_\lambda v_i}\right)
    \int r_n(x_n)^{1+\lambda}e^{-b_\lambda v_n}\pointDistribution(dx_n),
  \end{equation}
  and the last factor is at most one by
  \eqref{eq:coordinate-factor}.  After replacing it by one, the remaining
  integral has exactly the same form with \(n-1\) factors, and so we
  repeat the for the remaining coordinates. This proves that
  \eqref{eq:product-integral} is at most
  one, and integrating over $P_Y$ gives the final result. All
  integrals can be interchanged as the integrands are non-negative.
\end{proof}

\paragraph{Proof of \Cref{rdp:thm:app}}
Using the two key lemmas above, we prove the main theorem statement.

\paragraph{Bounding \(\sum_i v_i\).}
The bound \eqref{eq:corrected-moment} is useful only where
\(\sum_i v_i\) is small, so let
\begin{equation}\label{eq:bad-event}
  B:=\left\{\sum_{i=1}^n v_i>v_\star\right\}.
\end{equation}
Because \(\psi\) is increasing and concave (\Cref{lem:psi}), Jensen's
inequality gives
\begin{equation}\label{eq:spread-budget}
  \sum_{i=1}^n v_i
  =\frac{1}{2(\rdpOrder-1)^2}\sum_{i=1}^n\psi(J_i-1)
  \le\frac n{2(\rdpOrder-1)^2}
  \psi\!\left(\frac1n\sum_{i=1}^n(J_i-1)\right),
\end{equation}
so \(B\) is contained in the event that \(\frac1n\sum_i(J_i-1)\) exceeds
\((e^{(\rdpOrder-1)\rdpValue}-1)/\beta_1\).  Since \(J_i\ge1\),
\Cref{rdp:lem:spread} gives
\(\E_P\bigl[\frac1n\sum_i(J_i-1)\bigr]\le e^{(\rdpOrder-1)\rdpValue}-1\), so
Markov's inequality yields
\begin{equation}\label{eq:bad-probability}
  P(B)\le\beta_1.
\end{equation}
If \(\rdpValue=0\) then \(J_i=1\) almost surely for every \(i\), so \(B\) is
empty and \eqref{eq:bad-probability} holds trivially.

\paragraph{A tail bound for the information density.}
We split the tail bound into two parts:
\begin{equation}
  P\{\info>\kappa\}
  \le P(B)+P\bigl(B^c\cap\{\info>\kappa\}\bigr)
  \le\beta_1+P\bigl(B^c\cap\{\info>\kappa\}\bigr).
\end{equation}
It therefore suffices to choose \(\kappa\) making the second term at most
\(\beta_2\).

Fix \(0<\lambda<\rdpOrder-1\).
A Chernoff inequality bounds that
term by \(e^{-\lambda\kappa}\,\E_P[e^{\lambda\info}\mathbf
1_{B^c}]\).  On \(B^c\)  we have that
\(\sum_i v_i\le v_\star\), so \Cref{rdp:lem:corrected-moment} gives
\begin{equation}\label{eq:restricted-mgf}
  \E_P[e^{\lambda\info}\mathbf 1_{B^c}]
  \le e^{b_\lambda v_\star}
  \E_P\exp\left\{
    \lambda\info-b_\lambda\sum_{i=1}^n v_i
  \right\}
  \le e^{b_\lambda v_\star}.
\end{equation}
Hence
\(P(B^c\cap\{\info>\kappa\})\le e^{b_\lambda v_\star-\lambda\kappa}\),
which is at most \(\beta_2\) exactly when
\begin{equation}\label{eq:kappa-lambda}
  \kappa\ge\kappa_\lambda
  :=\frac{b_\lambda v_\star+L_2}{\lambda},
  \qquad L_2=\log(1/\beta_2).
\end{equation}

Taking \(\kappa=\kappa_\lambda\) in the display above,
\begin{equation}\label{eq:info-tail}
  P\{\info>\kappa_\lambda\}\le\beta_1+\beta_2=\beta,
  \qquad 0<\lambda<\rdpOrder-1.
\end{equation}

\paragraph{Conclusion.}
It remains only to choose \(\lambda\), which is the routine optimization
carried out in \Cref{lem:optimize}: the resulting value of
\(\inf_{0<\lambda<\rdpOrder-1}\kappa_\lambda\) is exactly the right-hand
side of \eqref{eq:optimized} with \(v_\star\) in place of \(v\).
Applying \eqref{eq:info-tail} at that \(\lambda\) and converting the
resulting tail bound on the information density into a max-information
bound \citep[Lemma 18]{dwork2015generalization} therefore proves
\(I_\infty^\beta(Y,S)\le v_\star+2\sqrt{v_\star L_2}+L_2/(\rdpOrder-1)\),
which is \eqref{eq:optimized} with \(v_\star\) in place of \(v\).  The remaining
claim \(v_\star\le v\) is \Cref{lem:psi}, and the right-hand side of
\eqref{eq:optimized} is increasing in \(v_\star\).\hfill\(\square\)

\subsubsection{Auxiliary Lemmas}\label{app:aux}

\begin{lemma}
  \label{lem:psi}
  The function \(\psi(u)=\log^2\bigl(\sqrt u+\sqrt{1+u}\bigr)\) is
  increasing and concave on \([0,\infty)\) and satisfies
  \(\psi(u)\le\frac14\log^2(4(1+u))\).  Consequently \(v_\star\le v\), for
  \(v_\star\) as in \eqref{eq:psi-v} and \(v\) as in \eqref{eq:v-def}.
\end{lemma}

\begin{proof}
  Differentiating gives
  \(\psi'(u)=\int_0^1[(1+u)(1+ut^2)]^{-1/2}\,dt\), whose integrand is
  nonnegative and decreasing in \(u\); hence \(\psi\) is increasing and
  concave.  Next, \(\sqrt u+\sqrt{1+u}\le2\sqrt{1+u}\) gives
  \(\psi(u)\le\frac14\log^2(4(1+u))\).  Finally \(\beta_1\le1\) gives
  \(1+(e^{(\rdpOrder-1)\rdpValue}-1)/\beta_1\le
  e^{(\rdpOrder-1)\rdpValue}/\beta_1\), so
  \begin{equation}
    v_\star
    \le\frac{n}{8(\rdpOrder-1)^2}
    \log^2\!\frac{4e^{(\rdpOrder-1)\rdpValue}}{\beta_1}
    =\frac n8\left(\rdpValue+\frac{\log(4/\beta_1)}{\rdpOrder-1}\right)^2
    =v. \qedhere
  \end{equation}
\end{proof}

\begin{lemma}
  \label{lem:optimize}
  Let \(v_\star\ge0\), and let \(\kappa_\lambda\) be as in
  \eqref{eq:kappa-lambda}.  Then
  \begin{equation}\label{eq:optimize}
    \inf_{0<\lambda<\rdpOrder-1}\kappa_\lambda
    =v_\star+2\sqrt{v_\star L_2}
    +\frac{L_2}{\rdpOrder-1}.
  \end{equation}
\end{lemma}

\begin{proof}
  This is elementary calculus.  Since
  \(b_\lambda/\lambda=1+\lambda/(1-\lambda/(\rdpOrder-1))\),
  \begin{equation}
    \kappa_\lambda
    =v_\star+\frac{(\rdpOrder-1)v_\star\lambda}{(\rdpOrder-1)-\lambda}
    +\frac{L_2}{\lambda},
  \end{equation}
  which is convex on \((0,\rdpOrder-1)\) with derivative
  \begin{equation}
    \frac{(\rdpOrder-1)^2v_\star}{\bigl((\rdpOrder-1)-\lambda\bigr)^2}
    -\frac{L_2}{\lambda^2}.
  \end{equation}
  For \(v_\star>0\) this vanishes exactly when
  \((\rdpOrder-1)\lambda\sqrt{v_\star}
  =\sqrt{L_2}\,\bigl((\rdpOrder-1)-\lambda\bigr)\), that is at
  \begin{equation}
    \lambda_*=\frac{(\rdpOrder-1)\sqrt{L_2}}
    {(\rdpOrder-1)\sqrt{v_\star}+\sqrt{L_2}},
    \qquad
    (\rdpOrder-1)-\lambda_*=\frac{(\rdpOrder-1)^2\sqrt{v_\star}}
    {(\rdpOrder-1)\sqrt{v_\star}+\sqrt{L_2}}.
  \end{equation}
  Substituting, the second and third terms of \(\kappa_{\lambda_*}\) each
  contribute \(\sqrt{v_\star L_2}\), the third leaving an
  additional \(L_2/(\rdpOrder-1)\); this is \eqref{eq:optimize}.
  For \(v_\star=0\) the second term vanishes identically and
  \(L_2/\lambda\downarrow L_2/(\rdpOrder-1)\) as
  \(\lambda\uparrow\rdpOrder-1\).
\end{proof}

\begin{lemma}[Hoeffding-type variance bound]\label{lem:hoeffding}
  Let \(V\) be a real random variable, let \(h>0\), and suppose
  \(A:=\E e^{hV}<\infty\) and \(B:=\E e^{-hV}<\infty\).  Then \(AB\ge1\), and,
  with \(\psi\) as in \eqref{eq:psi-v},
  \begin{equation}\label{eq:variance-fact}
    \Var(V)\le\frac{\psi(AB-1)}{h^2}.
  \end{equation}
\end{lemma}

This is a two-sided analogue of Hoeffding's lemma: rather than bounding the
variance of a \emph{bounded} random variable, it bounds the variance of
\(V\) from a pair of exponential moments, with no boundedness assumption.

\begin{proof}
  Cauchy--Schwarz gives \(AB\ge1\).  Choose
  \(a=(2h)^{-1}\log(A/B)\).  Then
  \begin{equation}
    \E e^{h(V-a)}=\E e^{-h(V-a)}=\sqrt{AB}.
  \end{equation}
  The function
  \begin{equation}
    q(u):=2\cosh(h\sqrt u)
    =2\sum_{k=0}^\infty\frac{h^{2k}u^k}{(2k)!},\qquad u\ge0,
  \end{equation}
  is increasing and convex.  Therefore Jensen's inequality gives
  \begin{equation}
    2\sqrt{AB}
    =\E q((V-a)^2)
    \ge q(\E(V-a)^2)
    \ge q(\Var(V)).
  \end{equation}
  Writing \(\sigma=\sqrt{\Var(V)}\), this reads
  \(e^{h\sigma}+e^{-h\sigma}\le2\sqrt{AB}\).  Solving this quadratic
  inequality for \(e^{h\sigma}\ge1\) yields
  \(e^{h\sigma}\le\sqrt{AB}+\sqrt{AB-1}\); taking logarithms and squaring
  proves \eqref{eq:variance-fact}.
\end{proof}

\begin{lemma}[Curvature of the log-Hellinger integral]\label{lem:g-curvature}
  Let \(\rdpOrder>1\), let \(w\) be a probability density with respect to
  \(\pointDistribution\) that is positive almost everywhere, let \(W\) be
  the probability measure with \(dW=w\,d\pointDistribution\), and suppose
  \(J(w)\le M<\infty\), with \(J\) as in \eqref{eq:J-def}.  Write
  \(Z:=\log w\), viewed as a random variable under \(W\), and
  \begin{equation}
    g(t):=\log\E_W e^{tZ}
    =\log\int w^{1+t}\,d\pointDistribution
    =\log \Hell_{1+t}(W\|\pointDistribution),
    \qquad -\rdpOrder\le t\le\rdpOrder-1,
  \end{equation}
  the logarithm of the Hellinger integral as a function of its order,
  equivalently \(g(t)=t\,\renyiDiv_{1+t}(W\|\pointDistribution)\).  Then
  \(M\ge1\),
  \begin{equation}\label{eq:g-normalization}
    g(-1)=g(0)=0,\qquad
    g'(0)=\KL(W\|\pointDistribution),\qquad
    g(\lambda)=\lambda\renyiDiv_{1+\lambda}(W\|\pointDistribution)
    \ (0<\lambda<\rdpOrder-1),
  \end{equation}
  and, with \(\psi\) as in \eqref{eq:psi-v},
  \begin{equation}\label{eq:curvature}
    g''(t)\le
    \begin{cases}
      \psi(M-1)/(\rdpOrder-1)^2, & -1\le t\le0,\\
      \psi(M-1)/(\rdpOrder-1-t)^2, & 0\le t<\rdpOrder-1.
    \end{cases}
  \end{equation}
\end{lemma}

\begin{proof}
  By \eqref{eq:J-def} and \(J(w)<\infty\), both
  \(\E_W e^{-\rdpOrder Z}\) and
  \(\E_W e^{(\rdpOrder-1)Z}\) are finite.  H\"older's inequality
  makes \(g(t)=\log\E_W e^{tZ}\) convex between these endpoints.
  On each compact subinterval of the interior, for every integer \(k\ge0\),
  \(\lvert Z\rvert^k e^{tZ}\) is bounded by a constant times
  \(e^{-\rdpOrder Z}+e^{(\rdpOrder-1)Z}\).  Differentiation under the
  expectation therefore shows that \(g\) is smooth in the interior.
  It also satisfies
  \begin{equation}
    g(-1)=g(0)=0,
    \qquad
    g(\rdpOrder-1)+g(-\rdpOrder)=\log J(w)\le\log M,
  \end{equation}
  the first two because \(w>0\) almost everywhere and
  \(\int w\,d\pointDistribution=1\), the third by the definition
  \eqref{eq:J-def} of \(J\).  Convexity implies
  \(g(\rdpOrder-1),g(-\rdpOrder)\ge0\), hence \(M\ge1\).  Two further
  identifications will be used: \(g'(0)=\E_W\log w=\KL(W\|\pointDistribution)\),
  which is the statement that \(\renyiDiv_\rdpOrder\to\KL\) as \(\rdpOrder\to1\)
  \citep[Theorem 5]{ervenRenyiDivergenceKullbackLeibler2014}, and
  \(g(\lambda)=\lambda\renyiDiv_{1+\lambda}(W\|\pointDistribution)\).  This
  proves \eqref{eq:g-normalization}.  Under the distribution tilted by
  \(e^{tZ-g(t)}\),
  \begin{equation}\label{eq:g-second}
    g''(t)=\Var_t(Z).
  \end{equation}
  We bound this variance by applying \Cref{lem:hoeffding} to \(Z\) under the
  tilted measure, with a suitable choice of \(h\) in each of the two ranges
  of \(t\).

  Under the distribution tilted by \(e^{tZ-g(t)}\),
  \begin{equation}\label{eq:tilted-product}
    \E_t e^{hZ}\,\E_t e^{-hZ}
    =\exp\{g(t+h)+g(t-h)-2g(t)\}.
  \end{equation}
  For \(-1\le t\le0\), take \(h=\rdpOrder-1\).  Increments of a convex function
  over a fixed interval length are nondecreasing in their starting point,
  so
  \begin{align*}
    g(t+\rdpOrder-1)-g(t)&\le g(\rdpOrder-1)-g(0)=g(\rdpOrder-1),\\
    g(t-\rdpOrder+1)-g(t)&\le g(-\rdpOrder)-g(-1)=g(-\rdpOrder).
  \end{align*}
  Their sum is at most \(\log M\).

  For \(0\le t<\rdpOrder-1\), take \(h=\rdpOrder-1-t\).  Then
  \(t+h=\rdpOrder-1\) and
  \(t-h=2t-\rdpOrder+1\in[-\rdpOrder,t]\).  A convex function attains
  its maximum on this
  interval at an endpoint, and \(g(t),g(-\rdpOrder)\ge0\); hence
  \begin{align*}
    g(t-h)-g(t)
    &\le \max\{g(-\rdpOrder),g(t)\}-g(t)\le g(-\rdpOrder),\\
    g(t+h)-g(t)
    &=g(\rdpOrder-1)-g(t)\le g(\rdpOrder-1).
  \end{align*}
  Again the exponent in \eqref{eq:tilted-product} is at most \(\log M\), so
  in both cases the two tilted exponential moments \(A=\E_te^{hZ}\),
  \(B=\E_te^{-hZ}\) satisfy \(AB\le M\).  By \Cref{lem:hoeffding},
  monotonicity of \(\psi\) (\Cref{lem:psi}), and \eqref{eq:g-second},
  \begin{equation}
    g''(t)=\Var_t(Z)\le\frac{\psi(AB-1)}{h^2}\le\frac{\psi(M-1)}{h^2},
  \end{equation}
  which is \eqref{eq:curvature} once \(h\) is replaced by \(\rdpOrder-1\) or
  by \(\rdpOrder-1-t\) in the two respective ranges.
\end{proof}

\begin{lemma}\label{lem:auxiliary}
  Let \(\rdpOrder>1\), let \(w\) be a probability density with respect to
  \(\pointDistribution\) that is positive almost everywhere, and let \(W\)
  be the probability measure with \(dW=w\,d\pointDistribution\).  If
  \(J(w)\le M<\infty\), with \(J\) as in \eqref{eq:J-def}, then \(M\ge1\).
  Put
  \begin{equation}
    v_w:=\frac{\psi(M-1)}{2(\rdpOrder-1)^2}.
  \end{equation}
  The Kullback--Leibler divergence satisfies
  \begin{equation}\label{eq:aux-mean}
    0\le\KL(W\|\pointDistribution)\le v_w,
  \end{equation}
  and for every \(0<\lambda<\rdpOrder-1\) the R\'enyi divergence exceeds it by
  at most
  \begin{equation}\label{eq:aux-centered}
    \lambda\bigl[\renyiDiv_{1+\lambda}(W\|\pointDistribution)
    -\KL(W\|\pointDistribution)\bigr]
    \le\psi(M-1)\left[-\log(1-\lambda/(\rdpOrder-1))-\lambda/(\rdpOrder-1)\right]
    \le\frac{v_w\lambda^2}{1-\lambda/(\rdpOrder-1)}.
  \end{equation}
  Consequently,
  \begin{equation}\label{eq:aux-uncentered}
    \log\int w^{1+\lambda}\,d\pointDistribution
    =\lambda\renyiDiv_{1+\lambda}(W\|\pointDistribution)
    \le v_w\left(\lambda+\frac{\lambda^2}{1-\lambda/(\rdpOrder-1)}\right).
  \end{equation}
\end{lemma}

The middle expression in \eqref{eq:aux-centered} is the centred cumulant
generating function of a Gamma variable, and the second inequality is the
usual relaxation of it: writing \(g(t):=\log\int w^{1+t}d\pointDistribution\)
for the log-Hellinger integral as a function of its order,
\eqref{eq:aux-centered} states that \(g\) is \emph{sub-gamma} with
variance factor \(2v_w\) and scale \(1/(\rdpOrder-1)\), in the sense of
\citet[\S2.4]{boucheron2013}.  Three cases are worth separating.  A
one-sided R\'enyi bound gives only
a bound on \(g(\rdpOrder-1)\), and hence subexponential
tails with no variance factor, which is what \Cref{thm:rdp-linear} uses;
requiring the bound at every order, as zCDP does, gives subgaussian tails
\citep{bun2016concentrated}.  A single order made two-sided sits between
the two: it yields a variance factor \(v_w\) that vanishes as \(M\to1\),
and this is what turns the linear dependence on \(n\) into \(\sqrt n\).
The function \(\psi=\arcsinh^2\!\sqrt{\cdot}\) belongs to the same family
as the \(\arcsinh\) in Prokhorov's exponential inequality for bounded
summands \citep{prokhorov1959extremal}.

\begin{proof}
  Apply \Cref{lem:g-curvature} to \(w\), with \(g\) and \(Z\) as there; this
  gives \(M\ge1\), the identifications \eqref{eq:g-normalization}, and the
  curvature bound \eqref{eq:curvature}.  It remains to integrate the latter.

  Because \(g(-1)=g(0)=0\), integration by parts gives
  \begin{equation}
    \KL(W\|\pointDistribution)=g'(0)
    =\int_{-1}^0(t+1)g''(t)\,dt
    \le\frac{\psi(M-1)}{2(\rdpOrder-1)^2}
    =v_w.
  \end{equation}
  The lower bound in \eqref{eq:aux-mean} is the nonnegativity of the
  Kullback--Leibler divergence.

  For \(0<\lambda<\rdpOrder-1\), Taylor's formula with integral remainder and
  \eqref{eq:curvature} give
  \begin{align*}
    \lambda\bigl[\renyiDiv_{1+\lambda}(W\|\pointDistribution)
    -\KL(W\|\pointDistribution)\bigr]
    &=g(\lambda)-\lambda g'(0)\\
    &=\int_0^\lambda(\lambda-t)g''(t)\,dt\\
    &\le\psi(M-1)
    \int_0^\lambda\frac{\lambda-t}{(\rdpOrder-1-t)^2}\,dt\\
    &=\psi(M-1)
    \left[-\log(1-\lambda/(\rdpOrder-1))-\lambda/(\rdpOrder-1)\right].
  \end{align*}
  For \(0\le u<1\),
  \begin{equation}
    -\log(1-u)-u
    =\sum_{k=2}^\infty\frac{u^k}{k}
    \le\frac{u^2}{2(1-u)}.
  \end{equation}
  Substituting \(u=\lambda/(\rdpOrder-1)\) proves \eqref{eq:aux-centered}.
  Adding \(\lambda\KL(W\|\pointDistribution)\le\lambda v_w\) to both sides
  gives \(g(\lambda)\le v_w(\lambda+\lambda^2/(1-\lambda/(\rdpOrder-1)))\),
  which is \eqref{eq:aux-uncentered}.
\end{proof}

\newcommand{\Law}{\operatorname{Law}}
\newcommand{\ret}{\widetilde P}
\newcommand{\rinfo}{\widetilde \imath(S,Y)}

\subsection{Proof of Theorem~\ref{thm:maxinf_from_epsdeltadp}}

In this section, we prove \Cref{thm:maxinf_from_epsdeltadp}. First, we state a stronger version of the upper bound for the $\beta \ge n \delta$ case, and then we provide the counterexample for $\beta < n\delta$. Combining these two results prove \Cref{thm:maxinf_from_epsdeltadp}.

\begin{theorem}\label{epsdelta:thm:app}
Assume that $\A$ is $(\epsilon,\delta)$-DP for some $\epsilon>0$ and $\delta \in [0,1]$, and $S$ is a dataset of $n$ \iid samples. For every $\beta > n\delta$, we have,
\begin{equation}\label{epsdelta:eq:bound}
 I_\infty^\beta(Y,S)
 \le \min\left\{n \epsilon, \min_{\substack{\lambda > 0}}\left\{\frac{n(1+\lambda)\epsilon^2}{8}
       +\frac{\log(1/\beta_0)}{\lambda}
       +\frac{n\delta}{\beta_0}c_\lambda \right\}, \right\},
\end{equation}
where $c_\lambda:=e^{2\epsilon}+\epsilon-1+e^{2\lambda\epsilon}$.
In particular, with
\begin{equation}\label{epsdelta:eq:lambda-star}
 \lambda_*:=\sqrt{\frac{8\log(1/\beta_0)}{n\epsilon^2}},
\end{equation}
we obtain the upper bound stated in \Cref{thm:maxinf_from_epsdeltadp}:

\begin{equation}\label{epsdelta:eq:explicit}
 I_\infty^\beta(Y,S)
 \le \min\left\{n \epsilon, \frac18n\epsilon^2 +\epsilon\sqrt{\frac n2 \log \frac{1}{\beta_0}} + \frac{n \delta}{\beta_0}c_{\lambda^*} \right\}.
\end{equation}

\end{theorem}

\paragraph{Proof of \Cref{epsdelta:thm:app}:}

Let
\begin{equation}
 \mathcal F_i:=\sigma(Y,X_1,\ldots,X_i),\qquad
 \mathcal F_0:=\sigma(Y),\qquad
 P_i:=P_{X_{\le i},Y},\qquad P_0:=P_Y.
\end{equation}

For finite measures \(U,V\), their common part \(U\wedge V\) has the
smaller of their densities relative to any common dominating measure.
For a fixed history \((y,x_{< i})\), define the capped
posterior-to-prior ratio
\begin{equation}\label{epsdelta:eq:cap}
 r_i^{y,x_{< i}}(x):=
 \frac{d[P_i\wedge e^\epsilon(P_{i-1}\otimes \mu)]}
      {d(P_{i-1}\otimes \mu)}(x_{< i},x,y).
\end{equation}
This retains the absolutely continuous part of the posterior and caps
its density at \(e^\epsilon\). Choose measurable versions, zero on
exceptional histories, so that every fixed-history function satisfies
\begin{equation}\label{epsdelta:eq:ratio-range}
 0\le r_i^{y,x_{< i}}\le e^\epsilon,
 \qquad \int r_i^{y,x_{< i}}(x)\,\mu(dx)\le1.
\end{equation}
When the history is random, write
\begin{equation}
 r_i(x):=r_i^{Y,X_{< i}}(x).
\end{equation}
This random function depends only on \(\mathcal F_{i-1}\).

\paragraph{The Retained Measure and Retained Information Density.} Starting with \(\widetilde P_0=P_Y\), retain the capped conditional
kernel at each coordinate:
\begin{equation}
 \widetilde P_i(dx_{\le i},dy)
 :=\widetilde P_{i-1}(dx_{< i},dy)
      r_i^{y,x_{< i}}(x_i)\,\mu(dx_i),
 \qquad \ret:=\widetilde P_n.
\end{equation}
Multiplying these kernels and dividing by \(Q=\mu^{\otimes n}\otimes P_Y\) gives
\begin{equation}\label{epsdelta:eq:factor}
 \frac{d\ret}{dQ}(x_1^n,y)
 =\prod_{i=1}^n r_i^{y,x_{< i}}(x_i).
\end{equation}
Consequently the retained information density is
\begin{equation}\label{epsdelta:eq:info}
 \rinfo:=\log\frac{d\ret}{dQ}(S,Y)
       =\sum_{i=1}^n\log r_i(X_i)
       \qquad\ret\text{-almost everywhere}.
\end{equation}

\paragraph{Privacy Excess} For any measurable \(w\) with \(0\le w\le e^\epsilon\) and
\(\int w\,d\mu\le1\), define its privacy excess by
\begin{equation}\label{epsdelta:eq:privacy-excess}
 \Delta_\epsilon(w):=\iint\bigl(w(x)-e^\epsilon w(x')\bigr)_+
                       \,\mu(dx)\mu(dx'),
 \qquad \Delta_i:=\Delta_\epsilon(r_i),
\end{equation}
where \(u_+:=\max\{u,0\}\). The quantity \(\Delta_\epsilon(w)\) measures how much
the candidate ratios exceed the pure-privacy factor \(e^\epsilon\).
In particular, \(\Delta_i\) depends only on \(\mathcal F_{i-1}\).

The measure \(\ret\) is not a probability measure and is not normalized, the capping discards some mass, however, in the following lemma we show that $\ret \le P$ and $\ret(\Omega) \ge 1-n\delta$.  Additionally, we show that for all $i$, the capped posterior-to-prior ratio $r_i$ has expected privacy excess at most $\delta$.

\begin{lemma}\label{epsdelta:lem:retain}
The retained measure satisfies
\begin{equation}\label{epsdelta:eq:retain}
 \ret\le P,\qquad \ret(\Omega)\ge1-n\delta,
 \qquad \E_P \Delta_i\le\delta\quad(i=1,\ldots,n).
\end{equation}
Here \(\Omega\) is the whole joint space.
\end{lemma}
\begin{proof}
Fix a coordinate \(i\) and a prefix \(u=x_{< i}\). For any measurable
set \(B\) of outputs, define
\begin{equation}
 K_x(B):=\int \mathcal A((u,x,v),B)\,\mu^{\otimes(n-i)}(dv),
 \qquad \bar K(B):=\int K_x(B)\,\mu(dx).
\end{equation}
Thus \(K_x(B)\) is the probability of \(Y\in B\) when the prefix
is \(u\) and the current observation is \(x\), averaging over the
remaining observations and the algorithm's randomness.
The law \(\bar K\) also averages over the current observation;
it is the output law given only the prefix.

Because \(X_i\) has prior law \(\mu\) given the prefix, Bayes' rule is
the following identity for the joint law of \((X_i,Y)\):
\begin{equation}
 \mu(dx)K_x(dy)
 =\bar K(dy)\,P(X_i\in dx\mid Y=y,X_{< i}=u).
\end{equation}
Let \(k_x\) be the density of the absolutely continuous part of
\(K_x\) relative to \(\bar K\). Taking densities in this identity gives
\begin{equation}
 \frac{d(P_{X_i\mid Y=y,X_{< i}=u})_{\mathrm{ac}}}{d\mu}(x)
 =k_x(y)
 \qquad (\mu\otimes\bar K)\text{-almost everywhere}.
\end{equation}
The subscript \(\mathrm{ac}\) means the part absolutely continuous
with respect to \(\mu\). Thus \(k_x(y)\) is also the posterior-to-prior
ratio for the observation \(x\). The retained ratio \(r_i\) caps it:
\begin{equation}
 r_i^{y,u}(x)=\min\{k_x(y),e^\epsilon\}.
\end{equation}
Equivalently, if we clip the output law by setting
\begin{equation}
 \widehat K_x:=K_x\wedge e^\epsilon\bar K,
 \qquad \widehat K_x(dy)=r_i^{y,u}(x)\,\bar K(dy),
\end{equation}
we obtain the same retained ratio as in \eqref{epsdelta:eq:cap}.
These identities hold almost everywhere and cover singular output
laws: the common part discards mass singular to \(\bar K\).

Privacy and averaging over \(x'\) give, for every output event \(B\),
\begin{equation}
 K_x(B)\le e^\epsilon K_{x'}(B)+\delta,
 \qquad K_x(B)\le e^\epsilon\bar K(B)+\delta.
\end{equation}
We use these two inequalities to control the lost mass and the excess.

\textbf{The measure $\ret$ losses at most $n\delta$ mass:}
The mass removed by the common part is the largest excess over an
output event. Hence
\begin{equation}
 1-\widehat K_x(\mathcal Y)
 =\sup_B\{K_x(B)-e^\epsilon\bar K(B)\}\le\delta,
\end{equation}
where \(\mathcal Y\) is the output space. Average over \(x\), then
over the prefix, to obtain
\begin{equation}\label{epsdelta:eq:loss}
 \int\left[1-\int r_i^{y,u}(x)\,\mu(dx)\right]
          P_{i-1}(du,dy)\le\delta.
\end{equation}
Each retained conditional kernel is dominated by the original one,
so induction gives \(\widetilde P_i\le P_i\).
The loss at step \(i\) is the left side of \eqref{epsdelta:eq:loss} integrated
against \(\widetilde P_{i-1}\) instead of \(P_{i-1}\).
The integrand is nonnegative, so this loss is also at most \(\delta\).
Summing the \(n\) losses gives \(\ret(\Omega)\ge1-n\delta\), and
the domination gives \(\ret\le P\).

\textbf{The expected excess is at most \(\delta\):}
Applying the same cap to two measures cannot increase their privacy
excess. To see this, take densities \(a,b,h\) of
\(K_x,K_{x'},e^\epsilon\bar K\) relative to their sum. Pointwise,
\begin{equation}
 \bigl(\min\{a,h\}-e^\epsilon\min\{b,h\}\bigr)_+
 \le(a-e^\epsilon b)_+.
\end{equation}
Indeed, a positive left side forces \(b<h\), so the second density
is unchanged by the cap, while the first can only decrease.
Integrating and using privacy therefore gives
\begin{equation}
 \int\bigl(r_i^{y,u}(x)-e^\epsilon r_i^{y,u}(x')\bigr)_+
             \,\bar K(dy)
 \le\sup_B\{K_x(B)-e^\epsilon K_{x'}(B)\}\le\delta.
\end{equation}
Now average over independent \(x,x'\sim \mu\). By the definition
of \(\Delta_i\) and Tonelli's theorem, the left side becomes
\(\E_P[\Delta_i\mid X_{< i}=u]\).
Averaging over the prefix proves \(\E_P\Delta_i\le\delta\).
\end{proof}

In the following lemma, we prove a moment bound for the retained information density (under $\ret$). We would use this moment bound later to prove a tail bound for $\rinfo$.
\begin{lemma}\label{epsdelta:lem:composition}
The retained information density \(\rinfo\) from \eqref{epsdelta:eq:info} is integrable
under \(\ret\). For every \(\lambda>0\),
\begin{equation}\label{epsdelta:eq:composed}
 \int\exp\left\{\lambda\rinfo
          -\lambda c_\lambda\sum_{i=1}^n \Delta_i\right\}\,d\ret
 \le\exp\left\{\frac{n\lambda(1+\lambda)\epsilon^2}{8}\right\}.
\end{equation}
\end{lemma}
\begin{proof}
Fix \(i\) and \((y,x_{< i})\), and suppress the history in the
notation \(r_i(x)=r_i^{y,x_{< i}}(x)\). Define the clipped ratio
\begin{equation}
 z_i(x):=\max\{r_i(x),e^{-\epsilon}\}
 =\min\{\max\{k_x(y),e^{-\epsilon}\},e^\epsilon\}
 \in[e^{-\epsilon},e^\epsilon].
\end{equation}
Here \(k_x\) is the output likelihood from Lemma~\ref{epsdelta:lem:retain};
the second equality holds almost everywhere. The lower clipping
changes the ratio used inside the logarithm, while mass is still
weighted by \(r_i\).

Put \(R=\int r_i\,d\mu\le1\). Give \(Z\) weight \(r_i(x)\mu(dx)\)
at \(\log z_i(x)\), and put the missing weight \(1-R\) at zero.
Explicitly, for every bounded measurable \(\varphi\),
\begin{equation}
 \E\varphi(Z)=\int r_i(x)\varphi(\log z_i(x))\,\mu(dx)
                      +(1-R)\varphi(0).
\end{equation}
Thus \(Z\in[-\epsilon,\epsilon]\). 

\textbf{Random Variable $Z$ satisfies the assumption needed for the moment bound in \Cref{epsdelta:lem:moment}:} Specifically we show the assumptions hold with $\Delta=\Delta_i$, \ie in two parts we show that $\E(Z-Z'-\epsilon)_+ \le \Delta_i$ and $1-\E e^{-Z} \le\Delta_i$.

\textit{Part 1.}
Let \(Z'\) be an independent copy of \(Z\).
Pairs involving the added zero have gap at most \(\epsilon\), so
\begin{equation}
 \E(Z-Z'-\epsilon)_+
 =\iint r_i(x)r_i(x')
       \left(\log\frac{z_i(x)}{e^\epsilon z_i(x')}\right)_+
                         \,\mu(dx)\mu(dx').
\end{equation}
We bound the integrand. Whenever the logarithm is positive,
\(z_i(x)>e^\epsilon z_i(x')\ge1\), so \(z_i(x)=r_i(x)\).
Using \((\log u)_+\le(u-1)_+\), we obtain
\begin{align*}
 &r_i(x)r_i(x')\left(\log\frac{z_i(x)}{e^\epsilon z_i(x')}\right)_+\\
 &\quad\le r_i(x)r_i(x')
                  \left(\frac{z_i(x)}{e^\epsilon z_i(x')}-1\right)_+\\
 &\quad=\frac{r_i(x)r_i(x')}{e^\epsilon z_i(x')}
                  \bigl(r_i(x)-e^\epsilon z_i(x')\bigr)_+\\
 &\quad\le\bigl(r_i(x)-e^\epsilon r_i(x')\bigr)_+.
\end{align*}
For the last step, \(r_i(x)\le e^\epsilon\) and
\(r_i(x')\le z_i(x')\) make the prefactor at most one;
replacing \(z_i(x')\) by the smaller \(r_i(x')\) can only increase
the positive part. Integrating over \(x,x'\) proves
\begin{equation}
 \E(Z-Z'-\epsilon)_+
 \le\iint\bigl(r_i(x)-e^\epsilon r_i(x')\bigr)_+\,\mu(dx)\mu(dx')
 =\Delta_i.
\end{equation}

\textit{Part 2.} The definition of \(Z\) gives
\begin{align*}
 \E e^{-Z}
 &=1-R+\int\frac{r_i(x)}{z_i(x)}\,\mu(dx)\\
 &=1-R+\int\min\{e^\epsilon r_i(x),1\}\,\mu(dx).
\end{align*}
The second equality follows by considering \(r_i(x)\ge e^{-\epsilon}\)
and \(r_i(x)<e^{-\epsilon}\). For each fixed \(x'\), since \(R\le1\),
\begin{align*}
 R-\min\{e^\epsilon r_i(x'),1\}
 &\le(R-e^\epsilon r_i(x'))_+\\
 &=\left(\int[r_i(x)-e^\epsilon r_i(x')]\,\mu(dx)\right)_+\\
 &\le\int\bigl(r_i(x)-e^\epsilon r_i(x')\bigr)_+\,\mu(dx).
\end{align*}
The last inequality is Jensen's inequality for the positive-part
function. Integrate over \(x'\) to conclude
\begin{equation}
 1-\E e^{-Z}
 =\int\bigl[R-\min\{e^\epsilon r_i(x'),1\}\bigr]\,\mu(dx')
 \le\Delta_i.
\end{equation}

\textbf{Apply the moment bound in Lemma~\ref{epsdelta:lem:moment}.}
Since \(r_i\le z_i\) and \(\lambda>0\),
\begin{align*}
 \int r_i^{1+\lambda}\,d\mu
 &\le\int r_i z_i^\lambda\,d\mu
 \le\E e^{\lambda Z}\\
 &\le\exp\left\{\frac{\lambda(1+\lambda)\epsilon^2}{8}
                       +\lambda c_\lambda\Delta_i\right\}.
\end{align*}
The second inequality adds the nonnegative mass \(1-R\) at zero.
Thus, at every fixed history,
\begin{equation}\label{epsdelta:eq:coordinate-factor}
 \int r_i(x)^{1+\lambda}e^{-\lambda c_\lambda\Delta_i}\,\mu(dx)
 \le\exp\left\{\frac{\lambda(1+\lambda)\epsilon^2}{8}\right\}.
\end{equation}

\textbf{Integrate the coordinates backwards.}
The product density \eqref{epsdelta:eq:factor} rewrites the left side of
\eqref{epsdelta:eq:composed} as
\begin{equation}
 \int\prod_{i=1}^n
 \left[r_i(X_i)^{1+\lambda}
       e^{-\lambda c_\lambda\Delta_i}\right]\,dQ.
\end{equation}
Under \(Q=\mu^{\otimes n}\otimes P_Y\), we can integrate each coordinate
against \(\mu\), with \(Y=y\) held fixed. Start with \(X_n\).
For \(i<n\), the \(i\)-th factor depends only on \(y,x_1,\ldots,x_i\),
so it is constant during this integration. Also \(\Delta_n\) depends
only on \(y,x_1^{n-1}\). The integral over \(x_n\) is therefore
\begin{equation}
 \left(\prod_{i=1}^{n-1}
       r_i(x_i)^{1+\lambda}e^{-\lambda c_\lambda\Delta_i}\right)
 \int r_n(x_n)^{1+\lambda}e^{-\lambda c_\lambda\Delta_n}\,\mu(dx_n).
\end{equation}
By \eqref{epsdelta:eq:coordinate-factor}, the last integral is at most
\(\exp\{\lambda(1+\lambda)\epsilon^2/8\}\), uniformly over the fixed
history. After replacing it by this constant, the remaining product
has exactly the same form with \(n-1\) factors.

Repeat for \(X_{n-1},\ldots,X_1\). Each step contributes the same
constant, so the full coordinate integral is at most
\(\exp\{n\lambda(1+\lambda)\epsilon^2/8\}\).
This bound is uniform in \(y\), and integrating over \(P_Y\) leaves
it unchanged. The factors are nonnegative, so Tonelli's theorem
justifies the successive integrations. Working backwards ensures
that the earlier factors stay fixed; it does not require the factors
themselves to be independent.

Finally, \eqref{epsdelta:eq:factor} and \eqref{epsdelta:eq:ratio-range} give
\(0\le d\ret/dQ\le e^{n\epsilon}\).
Since \(v|\log v|\), with value zero at zero, is bounded on
\([0,e^{n\epsilon}]\), the retained information density is
integrable under \(\ret\).
\end{proof}

\paragraph{Proof of \Cref{epsdelta:thm:app}:}

Fix \(\lambda>0\), and set
\begin{equation}\label{epsdelta:eq:threshold}
 \kappa:=\frac{n(1+\lambda)\epsilon^2}{8}
       +\frac{\log(1/\beta_0)}{\lambda}
       +\frac{n\delta}{\beta_0}c_\lambda.
\end{equation}
We show that \(\ret\{\rinfo>\kappa\}\le\beta_0\).

Suppose instead that \(E=\{\rinfo>\kappa\}\) has retained mass
\(p=\ret(E)>\beta_0\). Define the probability law on this event by
\begin{equation}\label{epsdelta:eq:R}
 R:=\frac{\ret(\,\cdot\cap E)}p.
\end{equation}
Since \(R\le\ret/p\le P/p\) and \(\Delta_i\ge0\),
Lemma~\ref{epsdelta:lem:retain} gives
\begin{equation}\label{epsdelta:eq:event-excess-budget}
 \E_R\sum_{i=1}^n\Delta_i
 \le\frac1p\sum_{i=1}^n\E_P\Delta_i
 \le\frac{n\delta}{p}.
\end{equation}
Restrict the moment bound \eqref{epsdelta:eq:composed} to \(E\) and apply
Jensen's inequality under \(R\):
\begin{align*}
 p\exp\left\{\lambda\E_R\rinfo
      -\lambda c_\lambda\E_R\sum_{i=1}^n\Delta_i\right\}
 &\le\int_E e^{\lambda\rinfo-\lambda c_\lambda\sum_i\Delta_i}\,d\ret\\
 &\le e^{n\lambda(1+\lambda)\epsilon^2/8}.
\end{align*}
But \(p>\beta_0\), \(\E_R\rinfo>\kappa\), and
\(\E_R\sum_i\Delta_i\le n\delta/\beta_0\). Hence the left side
strictly exceeds
\begin{equation}
 \beta_0\exp\left\{\lambda \kappa
             -\frac{\lambda n\delta}{\beta_0}c_\lambda\right\}
 =e^{n\lambda(1+\lambda)\epsilon^2/8},
\end{equation}
where the equality is the definition of \(\kappa\). This is a
contradiction. Integrability follows from
Lemma~\ref{epsdelta:lem:composition}. Thus the retained upper tail has mass
at most \(\beta_0\).

Now consider any measurable event $F$. Since $\ret \le P$, and $\ret(\Omega) \ge 1 - n\delta$, we have
\begin{align}
P(F) &\le \ret(F) + n\delta 
\\&= n\delta+\ret(F\cap E)+\ret(F\cap E^c)
\\&\le n\delta + \beta_0 + e^{\kappa} Q(F) 
\end{align}

Which, by definition of approximate max-information, shows for $\beta=\beta_0+n\delta$ we have 
\begin{align}
I_\infty^\beta(Y,S)\le \frac{n(1+\lambda)\epsilon^2}{8}
       +\frac{\log(1/\beta_0)}{\lambda}
       +\frac{n\delta}{\beta_0}c_\lambda. \label{eq:epsdeltabound1}
\end{align}

\paragraph{The bound \(n\epsilon\).}
Since \(r_i(X_i)\le e^\epsilon\) for every coordinate, the factorization
\eqref{epsdelta:eq:info} gives
\begin{equation}
 \rinfo=\sum_{i=1}^n\log r_i(X_i)\le n\epsilon
 \qquad\ret\text{-almost everywhere}.
\end{equation}
Thus \(\ret\{\rinfo>n\epsilon\}=0\), and for every
measurable event \(F\),
\begin{equation}
 P(F)\le\ret(F)+n\delta\le e^{n\epsilon}Q(F)+n\delta.
\end{equation}
Consequently, for $\beta \ge n \delta$:
\begin{align}
    I_\infty^\beta(Y,S)\le n \epsilon.\label{eq:epsdeltabound2}
\end{align}
Since the bound in \Cref{eq:epsdeltabound1} holds for every $\lambda > 0$, combining  \Cref{eq:epsdeltabound1} and  \Cref{eq:epsdeltabound2}  proves \Cref{epsdelta:thm:app}.
\hfill\(\square\)

\paragraph{The counterexample for the case $\beta < n\delta$.}

\begin{theorem}\label{thm:no_uniform_maxinfo_bound}
Let \(n\geq 1\), \(0<\delta<1\), and \(0<\beta<n\delta\).
Even for i.i.d.\ data and \((0,\delta)\) differentially private
mechanisms, \(I_\infty^\beta(Y,S)\) can be arbitrarily large.
Consequently, no finite upper bound depending only on
\(n,\epsilon,\delta,\beta\) can hold for all
\((\epsilon,\delta)\) differentially private mechanisms
throughout this parameter regime.
\end{theorem}

\begin{proof}
First assume  \(0<\delta<1/n\). Let \(\mu\) be the uniform distribution on \(\{1,\ldots,M\}\),
and let \(S=(X_1,\ldots,X_n)\sim \mu^{\otimes n}\).
Consider the mechanism
\begin{equation}
Y=\mathcal A(S)=
\begin{cases}
\bot, & \text{with probability }1-n\delta,\\
(J,X_J), & \text{with probability }n\delta,
\end{cases}
\end{equation}
where, conditional on disclosure, \(J\) is drawn uniformly
from \(\{1,\ldots,n\}\), independently of \(S\).

To verify privacy, consider two datasets differing only
in coordinate \(i\). Couple the two runs using the same
disclosure decision and index \(J\).
Their outputs agree unless disclosure occurs and \(J=i\),
an event of probability \(n\delta/n=\delta\).
Thus, for every measurable output event \(B\),
\begin{equation}
\Pr\{\mathcal A(s)\in B\}
\leq \Pr\{\mathcal A(s')\in B\}+\delta.
\end{equation}
Hence \(\mathcal A\) is \((0,\delta)\) differentially private,
and therefore also \((\epsilon,\delta)\) differentially
private for every \(\epsilon\geq 0\).

Write \(P=P_{S,Y}\) and \(Q=P_S\otimes P_Y\), and consider
the event
\begin{equation}
E:=\{(s,y):y=(j,s_j)\text{ for some }j\in\{1,\ldots,n\}\}.
\end{equation}
Under \(P\), this event occurs whenever the mechanism
discloses an observation, so \(P(E)=n\delta\).
Under \(Q\), the dataset and output are independent.
Conditional on a disclosed pair \((j,x)\), the probability
that \(X_j=x\) is \(1/M\). Therefore,
\begin{equation}
Q(E)=\frac{n\delta}{M}.
\end{equation}

By the definition of approximate max-information, any
admissible upper bound \(\kappa\) must satisfy
\begin{equation}
P(E)\leq e^\kappa Q(E)+\beta.
\end{equation}
Since \(\beta<n\delta\), it follows that
\begin{equation}
I_\infty^\beta(Y,S)
\geq \log\frac{P(E)-\beta}{Q(E)}
=\log M+\log\left(1-\frac{\beta}{n\delta}\right).
\end{equation}
For fixed \(n,\delta,\beta\), this lower bound tends to
infinity as \(M\to\infty\), proving the claim.

For $1/n \le \delta < 1$, set $\delta'=\frac{1+\beta}{2n}$, since  $0< \beta < 1$, then $\delta' < 1/n$ and $\beta < n \delta'$. Therefore, the construction above gives a $(0,\delta')$-DP algorithm with arbitrary large $I_\infty^\beta(Y,S)$. Since $\delta' < \frac{1}{n} \le \delta$, the algorithm is also $(0, \delta)$-DP. This extends the conclusion to the remaining range $1/n\leq\delta<1$.
\end{proof}

\subsubsection{Auxiliary moment bound}\label{epsdelta:app:auxiliary}
\begin{lemma}\label{epsdelta:lem:moment}
Let \(Z\in[-\epsilon,\epsilon]\), let \(Z'\) be an independent copy,
and let \(\Delta\ge0\).

\textbf{Centered bound.} If
\begin{equation}
 \E(Z-Z'-\epsilon)_+\le\Delta,
\end{equation}
then, for every \(t\in\mathbb R\),
\begin{equation}\label{epsdelta:eq:hoeffding}
 \log\E e^{t(Z-\E Z)}
 \le\frac{t^2\epsilon^2}{8}
       +|t|\bigl(e^{2\epsilon|t|}-1\bigr)\Delta.
\end{equation}
In particular, if \(\E Z=0\), the left side is \(\log\E e^{tZ}\).

\textbf{Uncentered bound.} If, in addition,
\begin{equation}
 1-\E e^{-Z}\le\Delta,
\end{equation}
then, with $c_\lambda:=e^{2\epsilon}+\epsilon-1+e^{2\lambda\epsilon}$, for every
\(\lambda>0\),
\begin{equation}\label{epsdelta:eq:one-step}
 \log\E e^{\lambda Z}
 \le\frac{\lambda(1+\lambda)\epsilon^2}{8}
       +\lambda c_\lambda\Delta.
\end{equation}
\end{lemma}

\begin{proof}[Proof of Lemma~\ref{epsdelta:lem:moment}]
We first prove the centered bound \eqref{epsdelta:eq:hoeffding} using only
the excess-gap assumption. We then use the additional normalization
assumption to control the mean and obtain \eqref{epsdelta:eq:one-step}.

\textbf{A variance estimate.}
For any random variable \(Z\) in an interval of length \(2\epsilon\)
and an independent copy \(Z'\), we first show
\begin{equation}
 \Var(Z)\le\frac{\epsilon^2}{4}
              +4\epsilon\,\E(Z-Z'-\epsilon)_+.
\end{equation}
Translate the interval to \([0,2\epsilon]\), and write
\(V=|Z-Z'|\), \(V_\circ=\min\{V,2\epsilon-V\}\).
Independence and symmetry give
\begin{align*}
 \Var(Z)-4\epsilon\,\E(Z-Z'-\epsilon)_+
 &=\tfrac12\E[V^2-4\epsilon(V-\epsilon)_+]\\
 &=\tfrac12\E V_\circ^2
 \le\tfrac\epsilon2\E V_\circ.
\end{align*}
For fixed \(z,z'\in[0,2\epsilon]\), the folded gap is
\begin{equation}
 \min\{|z-z'|,2\epsilon-|z-z'|\}
 =\int_0^\epsilon
   |\mathbf1_{[u,u+\epsilon]}(z)
      -\mathbf1_{[u,u+\epsilon]}(z')|\,du.
\end{equation}
The integral counts positions of a sliding interval that contains
exactly one of the two points.
Writing \(q_u=\Pr\{Z\in[u,u+\epsilon]\}\), Tonelli's theorem gives
\begin{equation}
 \E V_\circ=\int_0^\epsilon 2q_u(1-q_u)\,du
 \le\frac\epsilon2.
\end{equation}

This proves the variance estimate.

\textbf{Control the excess under exponential tilting.}
Return to \(Z\in[-\epsilon,\epsilon]\). It suffices to prove the
claim for \(t\ge0\): replacing \(Z\) by \(-Z\) leaves
\(\E(Z-Z'-\epsilon)_+\) unchanged, by exchangeability.
Write \(g(s)=\log\E e^{sZ}\), and let \(\E_s\) denote expectation
under the law tilted by \(e^{sZ-g(s)}\). For pairs, both copies are
independent under this tilted law.

Put \(T=(Z-Z'-\epsilon)_+\) and
\(h(z')=\E(Z-z'-\epsilon)_+\), with the latter expectation under
the original law. Since \(h\) is decreasing and \(z'\mapsto e^{sz'}\)
is increasing for \(s\ge0\), their covariance is nonpositive. Thus
\begin{equation}
 \E[e^{sZ'}h(Z')]
 \le\E e^{sZ'}\,\E h(Z')
 \le e^{g(s)}\Delta.
\end{equation}
Using \(Z\le\epsilon\) and \(g(s)\ge-s\epsilon\), we obtain
\begin{align*}
 \E_sT
 &=\frac{\E[e^{s(Z+Z')}T]}{e^{2g(s)}}\\
 &\le\frac{e^{s\epsilon}\E[e^{sZ'}h(Z')]}{e^{2g(s)}}
 \le e^{s\epsilon-g(s)}\Delta
 \le e^{2s\epsilon}\Delta.
\end{align*}
The variance estimate under the tilted law now gives
\begin{equation}
 g''(s)=\Var_s(Z)
 \le\frac{\epsilon^2}{4}+4\epsilon e^{2s\epsilon}\Delta.
\end{equation}
All differentiations are justified by boundedness of \(Z\).

\textbf{Integrate the variance bound.}
Since \(g(0)=0\) and \(g'(0)=\E Z\),
\begin{align*}
 \log\E e^{t(Z-\E Z)}
 &=\int_0^t(t-s)g''(s)\,ds\\
 &\le\frac{t^2\epsilon^2}{8}
       +\frac{\Delta}{\epsilon}
          \bigl(e^{2\epsilon t}-1-2\epsilon t\bigr).
\end{align*}
The trapezoid bound for the convex exponential function gives
\begin{equation}
 e^u-1=\int_0^u e^v\,dv\le\frac u2(1+e^u),
 \qquad
 e^u-1-u\le\frac u2(e^u-1),\quad u\ge0.
\end{equation}
Apply this with \(u=2\epsilon t\) to obtain
\eqref{epsdelta:eq:hoeffding}. The case \(t<0\) follows by the symmetry
noted above.

\textbf{Control the mean using the negative moment.}
Since \(\E e^{-Z}\ge e^{-\epsilon}\), the chord bound for
\(-\log v\) gives
\begin{equation}
 -\log\E e^{-Z}
 \le\frac{\epsilon}{1-e^{-\epsilon}}(1-\E e^{-Z})_+
 \le(1+\epsilon)\Delta.
\end{equation}
Here \(\epsilon/(1-e^{-\epsilon})\le1+\epsilon\) follows from
\(e^\epsilon\ge1+\epsilon\). Apply \eqref{epsdelta:eq:hoeffding}
at \(t=-1\):
\begin{align*}
 \E Z
 &\le-\log\E e^{-Z}+\frac{\epsilon^2}{8}
                         +(e^{2\epsilon}-1)\Delta\\
 &\le\frac{\epsilon^2}{8}+(e^{2\epsilon}+\epsilon)\Delta.
\end{align*}
\textbf{Control the positive moment.}
Apply \eqref{epsdelta:eq:hoeffding} at \(t=\lambda\):
\begin{align*}
 \log\E e^{\lambda Z}
 &=\lambda\E Z+\log\E e^{\lambda(Z-\E Z)}\\
 &\le\frac{\lambda(1+\lambda)\epsilon^2}{8}
       +\lambda\bigl(e^{2\epsilon}+\epsilon+e^{2\lambda\epsilon}-1\bigr)\Delta.
\end{align*}
This is \eqref{epsdelta:eq:one-step}.

\end{proof}

\newpage
\section{Experiments}\label{app:experiments}
Here we present the details of our model training experiments in \Cref{subsec:experiments}.

\paragraph{Implementation.}
For the experiments, we use Opacus \citep{opacus} for differentially private training and the implementation\footnote{\url{https://github.com/Diegogpcm/pacbayesgradients}} from \citep{garcia2025some, perez2021tighter} to directly optimize a PAC-Bayes bound for a family of Gaussian distributions. The results for data-independent priors are based on \cite{garcia2025some}, recomputed from the authors' released \texttt{results\_multiseed.ipynb} file in their repository. 

\paragraph{Datasets.}
We evaluate the bounds on the MNIST \citep{mnist} data with $60{,}000$ training samples, 
and CIFAR-10 \citep{cifar10} $50{,}000$ training samples. All reported guarantees are 
based on only the training data. 

\paragraph{Models.}
For MNIST, we use a fully-connected MLP with three 600-unit hidden layers and a 10-class output layer wit $1{,}198{,}210$ parameters. For CIFAR-10, we use three different models :1) a four-layer ConvNet with two convolutional layers followed by two fully connected layers with $9{,}946$ parameters, 2) a nine-layer ConvNet with six convolutional layers followed by three fully connected layers with $5{,}851{,}338$ parameters, and 3) a fifteen-layer ConvNet with twelve convolutional layers followed by three fully connected layers with $13{,}193{,}930$ parameters. So, as is typical for real-world deep networks, they are highly overparameterized. 

\paragraph{Prior and Posterior distributions.}
For all experiments, both the prior and posterior are Gaussian distributions. For the baseline, the independent prior is 
\begin{equation}    
\pi_0 = \mathcal{N}(\mu_0, \tau \mathrm{I} )
\end{equation}
with a randomly generated $\mu_0$, and $\tau$ is chosen from $\{0.001, 0.005, 0.01, 0.02, \dots, 0.08\}$, using a union-bound argument as described below. The DP-SGD prior is 
\begin{equation}
    \pi_{DP} = \mathcal{N}(\mu_{DP}, \tau \mathrm{I}),
\end{equation}
where $\mu_{DP}$ is trained with DP-SGD. The posterior is also a Gaussian distribution
\begin{equation}
    \rho = \mathcal{N}(\mu_\rho, \Sigma_\rho),
\end{equation}
with diagonal $\Sigma_\rho$. For all experiments, $\mu_\rho$ and $\Sigma_\rho$ are trained by optimizing the corresponding bounds. 

\paragraph{Training the DP-SGD prior.}
For training the prior, we train a deterministic model based on the DP-SGD algorithm in the streaming setting or with Poisson subsampling with multiple epochs.  

We use the following hyperparameters: Adam update rule, clipping norm $\zeta\in\{0.01, 0.02, 0.05, 0.1\}$, noise multiplier with $\sigma = 1 / \zeta$. We choose the learning rate from $\{0.001, 0.005, 0.01, 0.05, 0.1\}$. For streaming DP-SGD, we use a fixed batch size $m=5000$ for one epoch, and for Poisson subsampling we use expected batch size $m=5000$ (at each batch, each sample is included independently with probability $\frac mn$), and choose the number of epochs from $\{1, 2, 5, 10, 20\}$, \ie for each experiment at most $4\cdot 5\cdot 5 = 100$ tuples of hyperparameters are used. 

For streaming data, we compute the max-information using \Cref{thm:maxinf_dpsgd}. For Poisson subsampling, we convert 
the add/remove Rényi DP values computed by Opacus to replacement 
DP based on \cite{Mironov2017} as described below, 
use \Cref{rdp:thm:app} to compute max-information. 
We optimized the bound for $\alpha \in \{2, \dots, 5000\}$.

\paragraph{From add/remove RDP to replace-one RDP.}
Suppose $\A$ satisfies add/remove RDP with profile
$\rdpValue_{\mathrm{add}}(\gamma)$, for $\gamma>1$.
For a dataset of $n$ samples, write
\[
  Y_n(x_{1:n})=\A(x_1,\dots,x_n).
\]
Let $x_{1:n}$ and $x'_{1:n}$ differ only in coordinate $i$,
and denote their common dataset after deleting that coordinate by
\[
  x_{-i}=(x_1,\dots,x_{i-1},x_{i+1},\dots,x_n).
\]
Both $x_{1:n}$ and $x'_{1:n}$ are add/remove neighbors of
$x_{-i}$. Hence, the weak triangle inequality
\cite[Corollary~4(1)]{Mironov2017} gives, for every
$\rdpOrder>1$,
\begin{align*}
  \renyiDiv_\rdpOrder
  \bigl(Y_n(x_{1:n})\|Y_n(x'_{1:n})\bigr)
  &\leq
  \frac{2\rdpOrder-1}{2(\rdpOrder-1)}
  \renyiDiv_{2\rdpOrder}
  \bigl(Y_n(x_{1:n})\|\A(x_{-i})\bigr)\\
  &\quad+
  \renyiDiv_{2\rdpOrder-1}
  \bigl(\A(x_{-i})\|Y_n(x'_{1:n})\bigr)\\
  &\leq
  \frac{2\rdpOrder-1}{2(\rdpOrder-1)}
  \rdpValue_{\mathrm{add}}(2\rdpOrder)
  +
  \rdpValue_{\mathrm{add}}(2\rdpOrder-1).
\end{align*}
Thus, for every $n$, $\A$ satisfies replace-one RDP with profile
\[
  \rdpValue_{\mathrm{rep}}(\rdpOrder)
  =
  \frac{2\rdpOrder-1}{2(\rdpOrder-1)}
  \rdpValue_{\mathrm{add}}(2\rdpOrder)
  +
  \rdpValue_{\mathrm{add}}(2\rdpOrder-1).
\]

\paragraph{Computing the certificates.}
 To estimate an upper bound for the training error of the posterior, $\erisk$, instead of using $150{,}000$ Monte Carlo 
rollouts as in \cite{garcia2025some}, which requires heavy computations, we use the method described in \cite{biggs2022note}, and for each sample we sample different weights. 
This would save the amount of computation by the order of dataset size, \ie it needs $150000$ forward passes to get the same confidence instead of $150000 \times 50000$.
Based on the computed empirical risk, we apply the numeric inversion of $\kl$ to compute the upper bound on training risk based on a binary 
search, as discussed in Appendix~\ref{app:pacbayes}. Note that for each set of hyperparameters, this computation is done for the final posterior independent of the training procedure.

In the end, given the upper bound on training error and complexity terms, we apply another inversion of $\kl$ to compute the final upper bound on the posterior.

\paragraph{Hyperparameter Section via Union Bounds.}
To apply the bounds, the hyperparameters used to generate the prior 
must be data-independent. However, in practice, we want to select them 
from a finite grid based on the resulting certificates. 
To account for this selection, we employ a union bound argument. 
Formally, assume we have $K_1$ possible tuples of DP-SGD hyperparameters, 
and $K_2$ possible choices of the prior variance $\tau$. For the $i$-th 
DP-SGD hyperparameter tuple, let $\mu_i$ be the output of DP-SGD training 
and $\tau_j$ be the $j$-th possible value for $\tau$. 
Let $\pi_{i, j} = \mathcal{N}(\mu_i, \tau_j \mathbf{Id})$ be the prior 
generated by $i$-th hyperparameter tuple, and $j$-th value for $\tau$. 
Let $\kappa_i = I^{\beta/K_1}_{\infty}(\mu_i(S),S)$ be the corresponding 
max-information for $i$-th hyperparameter tuple. By the union bound 
argument, for any $\delta, \delta'$, and $\beta$, with probability at 
least $1 - (\delta - \delta')$ for all posteriors $\rho$, and all 
priors $\pi_{i, j}$, we have

\begin{equation}
    \risk(\rho) \leq \kl^{-1} \Biggl( \erisk(\rho)) \Biggm| \frac{\KL(\rho\|\pi_{i, j}) + \kappa_i + \log(\frac{2 K_1 K_2 \sqrt{n}}{\delta - \delta' - \beta})}{n}\Biggr), 
\end{equation}

Additionally, for any fixed posterior, with probability $1 - \frac{\delta'}{K_1 K_2}$ we have \citep{biggs2022note}
\begin{equation}
     \erisk(\rho) \le \kl^{-1}\Big( \tilde{R}_N(\rho)  \Big| \frac{\log(\frac{K_1 K_2}{\delta'})}{N} \Big),
\end{equation}
where $\tilde{R}_N(\rho)$ is the average training error draws of the posterior $\rho$ as described above.

Combining these two bounds with $\delta = 0.05, \delta' = 0.0025, \beta = 0.045, K_1 = 100, K_2 = 10$, results in the bound values we report in the paper.

\end{document}